\documentclass[journal]{IEEEtran}
\usepackage{array}
\usepackage{lineno}
\usepackage{graphicx}
\usepackage{epstopdf}
\usepackage{xcolor}
\usepackage{amssymb,amsfonts,amsmath,amscd,mathrsfs,bm}
\usepackage{fancyhdr}
\usepackage{marvosym}
\usepackage{amsthm}
\usepackage{amsmath}
\usepackage{bm}
\usepackage{booktabs}
\usepackage{longtable}
\usepackage{multirow}
\usepackage{makecell}
\usepackage{hhline}
\usepackage{stmaryrd}
\usepackage{color, xcolor}
\usepackage{dcolumn}
\usepackage{algorithm}
\usepackage{algorithmic}
\usepackage{lipsum}
\usepackage[mathscr]{eucal}
\usepackage{txfonts}
\usepackage{mathtools}
\usepackage{booktabs}
\usepackage{cite}
\usepackage{subfigure}
\usepackage{multicol}
\usepackage{balance}
\usepackage{float}
\usepackage{hyperref}
\usepackage{gensymb}
\usepackage{verbatim}
\usepackage{setspace}
\usepackage{ulem}
\usepackage{blindtext}
\usepackage{diagbox}
\usepackage{afterpage}
\usepackage{placeins}
\makeatletter 

\newcommand{\Rmnum}[1]{\expandafter\@slowromancap\romannumeral #1@}
\makeatother

\begin{document}
\newcommand{\tabincell}[2]{\begin{tabular}{@{}#1@{}}#2\end{tabular}}

\title{Physically Informed Synchronic-adaptive Learning for Industrial Systems Modeling in Heterogeneous Media with Unavailable Time-varying  Interface}

\author{Aina Wang, 
        Pan Qin, 
        Xi-Ming Sun,~\IEEEmembership{Senior Member,~IEEE}}

\maketitle

{\begin{abstract}
Partial differential equations (PDEs) are commonly employed to model complex industrial systems characterized by multivariable dependence. Existing physics-informed neural networks (PINNs) excel in solving PDEs in a homogeneous medium.
However, their feasibility is diminished when  PDE parameters are  unknown due to a lack of physical attributions and time-varying interface is unavailable arising from heterogeneous media.
To this end, we propose a data-physics-hybrid method, physically informed synchronic-adaptive learning (PISAL), to solve PDEs for industrial systems modeling in heterogeneous media.
First, $Net_1$, $Net_2$, and $Net_I$, are constructed to approximate the solutions satisfying PDEs and the  interface.
$Net_1$ and $Net_2$ are utilized to synchronously learn each solution satisfying PDEs with diverse parameters, while 
$Net_I$ is employed to adaptively learn the unavailable time-varying interface.
Then, a criterion combined with $Net_I$ is introduced to adaptively distinguish the attributions of measurements and collocation points.
Furthermore, $Net_I$ is integrated into a data-physics-hybrid loss function. 
Accordingly, a synchronic-adaptive learning (SAL) strategy is proposed to decompose and optimize  each subdomain.
Besides,  we theoretically prove the approximation capability of PISAL.   
Extensive experimental results verify that the proposed PISAL can be used for industrial systems modeling in heterogeneous media, which faces the challenges of lack of physical attributions and unavailable time-varying interface.
\end{abstract}
\def\abstractname{Note to Practitioners}
\begin{abstract}
The motivation behind this paper is to devise a method for industrial systems modeling, even in cases where unknown PDE parameters are caused by a lack of physical attributions and an unavailable  time-varying interface is caused by heterogeneous media. 
Existing methods apply domain decomposition technique under the assumption that the interface is available.
To this end,  a data-physics-hybrid method, PISAL, in which $Net_1$, $Net_2$,  and $Net_I$  with SAL strategy are proposed to solve PDEs for industrial system modeling in heterogeneous media.
$Net_1$, $Net_2$, and $Net_I$ are first constructed to learn the solutions satisfying PDEs and the time-varying interface.
$Net_1$ and $Net_2$ are for synchronously learning each solution satisfying PDEs with diverse parameters; $Net_I$ is for adaptively learning the unavailable  time-varying interface.  
Subsequently, a criterion combined with $Net_I$  is introduced, which is to adaptively distinguish different physical attributions of measurements and collocation points. 
Additionally, $Net_I$ is integrated into a data-physics-hybrid loss function.  
Accordingly, a SAL strategy is proposed to decompose and optimize  each subdomain. 
Besides, we theoretically prove the approximation capability of PISAL.   
To validate the efficacy of the proposed PISAL,  the classical two-phase Stefan problem and the mixed Navier-Stokes problem are employed.
Results highlight the feasibility of our method in industrial systems modeling with the abovementioned challenges and demonstrate some comparisons with relevant state-of-the-art approaches.
Thus, the proposed method is suitable for industrial systems modeling applications.
In future research, we intend to conduct the proposed PISAL on an experimental platform to further verify the feasibility in practical scenarios.
\end{abstract}}

\IEEEpeerreviewmaketitle

\section{Introduction}\label{se1}

\IEEEPARstart
{D}{ifferential} equations are often used to model industrial systems.
Achievements have been made based on ordinary differential equations (ODEs), which focus on time series data\cite{10304675}.
However,  as far as spatiotemporal dependence is involved in industrial systems, ODEs cannot easily handle such problems.
Partial differential equations (PDEs) are one of the mathematical models for describing spatiotemporal dependence in various areas\cite{9516587}, such as thermal processes \cite{yu2022hybrid}, transport processes\cite{feng2023computation}, and vibration processes\cite{ji2022vibration}.
 Thus, PDEs greatly facilitate large-scale system control, monitoring, and other applications\cite{huang2023physical,5345678}.
Note that numerous real industrial systems concurrently exist in multiple physical domains with time-varying interfaces, which pose a challenge for industrial system modeling.  Such as the phase transitions in matter \cite{9852677,10234210}. Furthermore, the challenges of identifying PDE parameters posed by the lack of physical attributions are inevitable. These challenges are commonly referred to PDE inverse problems in heterogeneous media, encompassing areas such as medical sensing and imaging \cite{9462124,9695519}, as well as thermal management\cite{wang2019spatiotemporal}. 
Consequently, the methods for solving PDEs play a vital role in identifying PDE parameters within heterogeneous media.

Several classical numerical methods for solving PDEs have been successfully used in numerous fields, such as the finite difference method\cite{5345678} and finite element method\cite{4049770}.
However, efficiency and accuracy cannot be concurrently balanced well because of their mesh generation.
With the development of information technologies,
industrial data can be obtained by using a large number of sensors\cite{9423982}. 
Thus, data-driven approaches have made great achievements in industrial applications \cite{sun2021survey,10268910}.
Among data-driven methods, deep learning methods have numerous successful applications and well-known approximation ability \cite{sirignano2018dgm, lu2021deepxde}.
Thus, they have been motivated to applications of PDE solving.
However, the equipment of hard sensors is limited by the environment,  which poses a challenge for data acquisition\cite{yao2021figan, lui2022supervised, 10061539}.
Physics-informed neural networks (PINNs)\cite{raissi2019physics} provides a transformative direction for solving PDEs with 
small data regimes and physics-informed knowledge.
Consequently,  they are rapidly developed for solving PDEs, which are considered as a data-physics-hybrid method.
 \cite{9931982} proposed end-to-end PINNs for defect identification and 3-D reconstruction for metal structures. 
\cite{huang2022physics} proposed a physics-informed time-aware neural network method for solving industrial nonintrusive load monitoring problems.
\cite{jagtap2022physics} employed PINNs in supersonic flows. 
The abovementioned works focus on the applications in a homogeneous medium.
However, those methods cannot easily solve PDEs within heterogeneous media, which can happen in heat conduction in multilayer media\cite{zhang2022multi},  motion in different viscous fluid substances\cite{alhubail2022extended}, and transport modeling in spatial variability within heterogeneous systems\cite{aliakbari2023ensemble}. 

Several effective contributions have been made to  PDE solving  within heterogeneous media.
\cite{zhang2022multi} adopted the domain decomposition technique  with available interfaces  to divide the multilayer media into several subdomains and approximate the field on each subdomain with  fully connected neural networks 
(FCNNs).
\cite{alhubail2022extended} solved fluid flow in heterogeneous media using the extended PINNs, in which the domain decomposition technique divided the original heterogeneous domain into several homogeneous subdomains.
\cite{aliakbari2023ensemble} proposed an improved framework for transport modeling in heterogeneous media.
The abovementioned works made great improvements in  PDE solving in heterogeneous media with the assumption that the interface between neighboring domains is available. However, the interface cannot always be easily obtained in practice, such as radio wave nondestructive testing\cite{9632211}.
A prototype inverse Stefan problem with dynamic interactions between different material phases was considered, in which 
an indicator function was ensembled in original PINNs (indicator-PINNs)\cite{cai2021physics}. 
However, the parameters of PDEs are not fixed in the whole domain and will jump at the media interface.
\cite{li2021schwarz} proposed a Schwarz-type iterative method to decouple the different physical processes in each subdomain occupied by a single fluid, which allows solving in each subdomain with a single physical process described by a Stokes problem. However,  the transmission conditions implemented could cause the computational costs of the iterative method.
Therefore, the method for solving PDEs, which can be used for industrial systems modeling  applications within heterogeneous media facing the challenges of unknown physical attributions and unavailable time-varying interface,  should be further investigated.

To address the challenges of unknown PDE parameters caused by a lack of physical attributions and unavailable time-varying interface caused by heterogeneous media, 
this article proposes a data-physics-hybrid modeling method, physically informed synchronic-adaptive learning (PISAL).
First,  $Net_1$, $Net_2$, and $Net_I$ are proposed for approximating the solutions satisfying PDEs and the time-varying interface.    
$Net_1$ and $Net_2$ are for synchronously learning  each solution satisfying PDEs with diverse parameters;  $Net_I$  is for adaptively learning the time-varying interfaces.    
Then,  a criterion combined with $Net_I$ is proposed to adaptively distinguish the attributions of measurements and collocation points.
Furthermore, $Net_I$ is integrated into a data-physics-hybrid loss function. 
Accordingly, a synchronic-adaptive learning (SAL) strategy is proposed to decompose and optimize each subdomain.
Besides, the approximation capability of PISAL is theoretically proved.   
Finally, the classical two-phase Stefan problem and mixed Navier-Stokes problem are used to validate the effectiveness and the feasibility of the proposed PISAL, as well as some comparisons with relevant state-of-the-art methods are conducted.

The main contributions of this article can be summarized
as follows.
\begin{itemize}
\item[1)]
PISAL is a data-physics-hybrid method for industrial systems modeling to address the challenges posed by unknown PDE parameters  due to a lack of physical attributions and the unavailable  time-varying interface  due to heterogeneous media.
\item[2)]
We theoretically prove that the proposed PISAL can approximate the whole field with unknown physical attributions and an unavailable time-varying interface.
\item[3)]
The two-phase Stefan problem and the mixed Navier-Stokes problem are employed to validate the effectiveness and feasibility of our proposed method.
\end{itemize}

The rest of this article is organized as follows.
Section \ref{section:2} briefly reviews PINNs. 
Section \ref{section:Methods}
introduces our proposed approach for industrial systems modeling in a heterogeneous media, PISAL,  
including the main structure, decomposition criterion, learning strategy, and theoretical justification.
Section \ref{Illusive Example} conducts two classical examples to validate the accuracy and efficiency of the proposed method and provides discussions.
In Section \ref{Conclusion},
some concluding remarks are given.

\section{Preliminaries}\label{section:2}
\subsection{Brief Overview for PINNs}
We  first briefly review the basic idea of PINNs, in which the following PDEs are considered:\begin{equation}\label{eq:23111903}
\hspace{-1mm}\boldsymbol{u}_{h_t}(\boldsymbol{x}, t)\hspace{-0.6mm}+\hspace{-0.6mm}\mathcal{F}[\boldsymbol{u}_h(\boldsymbol{x}, t)]\hspace{-0.6mm}=\hspace{-0.6mm}\boldsymbol{g}_h(\boldsymbol{x}, t), \boldsymbol{x} \in \Omega \subseteq \mathbb{R}^{d}, t \in[0, T] \hspace{-0.6mm}\subset \hspace{-0.6mm}\mathbb{R}.
\end{equation}
Note that \eqref{eq:23111903} can be used to model industrial systems with multivariable dependence in a homogenesis medium.
Here, $\boldsymbol{x}$ is the spatial variable;
$t$ is the temporal variable;
the  domain $\Omega \subseteq \mathbb{R}^{d}$ is a spatial bounded open set with the smooth $(C^{1})$ boundary $\partial\Omega$.
Here, let $\Omega\cup\partial\Omega=\bar\Omega\subseteq\mathbb{R}^d$ denote the closure of $\Omega$.
Let $[0,T]\times\bar\Omega=\mathbb{D}\subset\mathbb{R}^{d+1}$ be the underlying space-time domain.
$\boldsymbol{u}_h: \mathbb{R}^{d}\times\mathbb{R}\rightarrow \mathbb{R}^{m}$  is the  solution with respect to spatiotemporal variables, $\boldsymbol{u}_h \in X$;
$\boldsymbol{g}_h: \mathbb{R}^{d}\times\mathbb{R}\rightarrow \mathbb{R}^{m}$ is the  source, $\boldsymbol{g}_h\in Y$; 
Let $X=L^{1}\left(\mathbb{D} ; \mathbb{R}^{m}\right), Y=L^{1}\left(\mathbb{D} ; \mathbb{R}^{m}\right)$  be underlying functional spaces;
$\mathcal{F}$ is a series of differential operators.

PINNs  can be  trained by minimizing the loss function
\begin{equation} \label{eq:23111904}
{{\rm {MSE}}_H} = {\rm {MSE}}_{DH}+ {\rm {MSE}}_{PH}.
\end{equation}
Here, ${\rm MSE}_{DH}$ is formulated as the following
\begin{equation}\label{eq:23111905}
{\rm MSE}_{DH}= \displaystyle \frac{1}{{\rm card}\left(D\right)}\underset{{(\boldsymbol{x}, t,u)\in D}}{\sum} \left(\hat{\boldsymbol{u}}_h\left(\boldsymbol{x},t;\boldsymbol{\Theta}_{U_h}\right) -{\boldsymbol{u}_h\left(\boldsymbol{x},t\right)}\right)^{2},
\end{equation}
where $D$ is the training dataset and  ${\rm card}\left(D\right)$ is the cardinality of $D$ ( Let card $(\cdot)$ denote the cardinality of set $\cdot$ in this article); 
$\hat {\boldsymbol{u}}_h\left(\boldsymbol{x},t;\boldsymbol{\Theta}_{U_h}\right)$ is the function of PINNs  to approximate the solution satisfying  \eqref{eq:23111903}, with $\boldsymbol{\Theta}_{U_h}$ being a set of parameters.
This mean-squared-error (MSE) term
\eqref{eq:23111905} is considered as a data-driven loss.

A residual function for \eqref{eq:23111903} is defined as the following
\begin{equation}\label{eq:forward2}
\boldsymbol{f}_h (\boldsymbol{x},t) = \boldsymbol{u}_{h_t}(\boldsymbol{x}, t)+\mathcal{F}[\boldsymbol{u}_h(\boldsymbol{x}, t)]-\boldsymbol{g}_h(\boldsymbol{x}, t).
\end{equation}
Consequently, ${\rm {MSE}}_{PH}$ is formulated as 
\begin{equation}\label{eq:20240120-wan01}
{\rm MSE}_{PH}= \displaystyle \frac{1}{{\rm card}\left(E\right)} \underset{(\boldsymbol{x}, t)\in {E}}{\sum} \boldsymbol{\hat f}_h\left(\boldsymbol{x}, t;\boldsymbol{\Theta}_{U_h}\right)^2,
\end{equation}
where $E$ is the set of collocation points; $\boldsymbol{\hat f}_h \left(\boldsymbol{x}, t;\boldsymbol{\Theta}_{U_h}\right)$ is an estimation to the residual function $\boldsymbol{f}_h(\boldsymbol{x},t)$ \eqref{eq:forward2} based on $\hat{\boldsymbol{u}}_h\left(\boldsymbol{x},t;\boldsymbol{\Theta}_{U_h}\right) $.
$\boldsymbol{\hat f}_h \left(\boldsymbol{x}, t;\boldsymbol{\Theta}_{U_h}\right)$ is obtained as the following
\begin{equation}\label{eq:f}
\hspace{-1.2mm}\boldsymbol{\hat f}_h \left(\boldsymbol{x}, t;\boldsymbol{\Theta}_{U_h}\right) \hspace{-0.5mm}=\hspace{-0.5mm}\hat{\boldsymbol{u}}_{h_t}\left(\boldsymbol{x},t;\boldsymbol{\Theta}_{U_h}\right) \hspace{-0.3mm}+\hspace{-0.3mm}\mathcal{F}\hspace{-0.3mm}\left[\boldsymbol{\hat u}_h\left(\boldsymbol{x},t;\boldsymbol{\Theta}_{U_h}\right) \right]\hspace{-0.3mm}-\hspace{-0.3mm}\boldsymbol{g}_h(\boldsymbol{x}, t).\hspace{-0.3mm}
\end{equation} 
MSE$_{PH}$ \eqref{eq:20240120-wan01} is used to regularize $\boldsymbol{\hat u}_h\left(\boldsymbol{x},t;\boldsymbol{\Theta}_{U_h}\right) $ to satisfy  \eqref{eq:23111903}, 
which is considered as a physics-informed loss.
Readers are referred to \cite{raissi2019physics} for the details.

Note that  PDE in \eqref{eq:23111903} is well-posed. Accordingly,  
$\left\Vert\ \hat{\boldsymbol{u}}_{h_t}\left(\boldsymbol{x},t;\hat{\boldsymbol{\Theta}}_{U_h}\right)-\boldsymbol{u}_{h_t}\left(\boldsymbol{x},t\right)\right\Vert<\varepsilon$ for all $\varepsilon>0$ can be obtained according to our previous work \cite{wang2023coupled} with $X=L^{1}\left(\mathbb{D} ; \mathbb{R}^{m}\right)$ and $Y=L^{1}\left(\mathbb{D} ; \mathbb{R}^{m}\right)$.
\newtheorem{remark}{\bf Remark}
\begin{remark}\label{Rem:24011201}
Note that the property of $X$ with respect to completeness is available according to $L^{1}$ being Banach.
 $\hat{\boldsymbol{u}}_h\left(\boldsymbol{x},t;\boldsymbol{\Theta}_{U_h}\right)$ is the function of PINNs, with $\boldsymbol{\Theta}_{U_h}$ being a set of parameters.
Thus,
for all exact solutions $\boldsymbol{u}_h$ satisfying \eqref{eq:23111903},
there exists a series of functions $\left\{\hat{\boldsymbol{u}}_h\left(\boldsymbol{x},t;\boldsymbol{\Theta}_{U_h}^{(j)}\right)\right\}=\left\{\hat{\boldsymbol{u}}_h\left(\boldsymbol{x},t;\boldsymbol{\Theta}_{U_h}^{(j)}\right)\in X| j=1,2,\cdots\right\}$ to ensure $\left\Vert\hat{\boldsymbol{u}}_{h_t}\left(\boldsymbol{x},t;\boldsymbol{\Theta}_{U_h}^{(j)}\right)-\boldsymbol{u}_{h_t}\left(\boldsymbol{x},t\right)\right\Vert\rightarrow0\ (j\rightarrow  \infty)$.
Consequently, if PINNs is well-trained to achieve $\hat{\boldsymbol{\Theta}}_{U_h}^{(j)}$ using iterative optimization methods to ensure $\hat{\boldsymbol{u}}_{h}\left(\boldsymbol{x},t;\hat{\boldsymbol{\Theta}}_{U_h}^{(j)}\right)$ satisfying \eqref{eq:23111903}, PINNs can be used for industrial system modeling applications within a homogenesis medium.
\end{remark}
\section{Methods}\label{section:Methods}
\subsection{PDEs for Modeling Industrial Systems  in Heterogeneous Media}
In this article, we consider PDEs for modeling industrial systems in  heterogeneous media are formulated as the following
\begin{subequations}\label{eq:231120-6}
\begin{align}
&\hspace{-3mm} \boldsymbol{u}_{1t}(\boldsymbol{x}, t)+\mathcal{N}[\boldsymbol{u}_1(\boldsymbol{x}, t);{\lambda}_{1}]=\boldsymbol{g}_1(\boldsymbol{x}, t),\hspace{-1mm}\ (\boldsymbol{x},t)\in \Omega_1\times [0, T], \label{eq:231120-01}
\\
&\hspace{-3mm} \boldsymbol{u}_{2t}(\boldsymbol{x}, t)+\mathcal{N}[\boldsymbol{u}_2(\boldsymbol{x}, t);{\lambda}_{2}]=\boldsymbol{g}_2(\boldsymbol{x}, t),\hspace{-1mm} \ (\boldsymbol{x},t)\in \Omega_2	\times [0, T]. \label{eq:231120-02}
\end{align}
\end{subequations}
Here, $\Omega=\Omega_1\cup\Omega_2$; the interiors of $\Omega_1$ and $\Omega_1$ satisfy $\text{int }  \Omega_2\cap \text{int } \Omega_2=\varnothing$; the interface $I=\partial\Omega_1\cap\partial\Omega_2$;
$\Omega_1$ and $\Omega_2$ with diverse physical attributions.
Thus, the PDE parameters 
$\lambda_1$ and $\lambda_2$ in the whole domain with $\lambda_1\neq \lambda_2$.
Note that $\Omega_i$ for $i=1,2$ are not fixed, which shows variation with respect to $t$.
Accordingly, the interface $I$ is time-varying.
Let $i$ denote the index of the $i^{\rm {th}}$ subdomain in this article;
$\boldsymbol{u}_i: \mathbb{R}^{d}\times\mathbb{R}\rightarrow \mathbb{R}^{m}$ denotes a  solution in space-time domain $\mathbb{D}_i$, $\boldsymbol{u}_i \in X_i$;
$\boldsymbol{g}_i: \mathbb{R}^{d}\times\mathbb{R}\rightarrow \mathbb{R}^{m}$ denotes a  source, $\boldsymbol{g}_i \in Z_i$;
Let $X_i=L^{1}\left(\mathbb{D}_i ; \mathbb{R}^{m}\right) $, $Z_i=L^{1}\left(\mathbb{D}_i^{'} ; \mathbb{R}^{m}\right) $ be underlying functional space for any observation subdomain $\mathbb{D}_{i}^{'}$, which is dense in space-time domian $\mathbb{D}_{i}$;
$\boldsymbol{\mathcal{N}}[\cdot ]$ is a series of  restriction operators. 
The other variables are similar to \eqref{eq:23111903}.
The exact physical attributions, i.e. the PDE parameters $\lambda_1$ and $\lambda_2$ are unknown in this article by considering the practical situations, such as a lack of thorough understanding of complex thermal systems\cite{wang2019spatiotemporal}.
\begin{remark}\label{Rem:24011202}
Note that the unknown $\lambda_i$ leads to \eqref{eq:231120-6}
will be ill-posed and the unique $\boldsymbol{u}_i$ to \eqref{eq:231120-6} cannot be guaranteed  in $\mathbb{D}_i$, where $i=1,2$.
Furthermore,  the observation subdomain $\mathbb{D}_{i}^{'}$ is dense in space-time domain $\mathbb{D}_i$.
Thus,
for all exact solutions $\forall\boldsymbol{u}_{i} \in \mathbb{D}_i$ satisfying \eqref{eq:231120-6}, there exists a set of measurements $\left\{\boldsymbol{u}_{i}^{(n)} \in \mathbb{D}^{'}_{i} | n=1,2,\cdots\right\}$ to ensure $\displaystyle\lim_{n\rightarrow\infty}\left\|\boldsymbol{u}_{i}^{(n)}-\boldsymbol{u}_{i}\right\| \rightarrow 0$.
Consequently, solving the inverse problem \eqref{eq:231120-6} in the whole space-time domain $\mathbb{D}$, consider the measurement $\boldsymbol{u}_i$ in the observation subdomain $\mathbb{D}^{'}_{i}\subset \mathbb{D}$.
The theory with respect to well-posedness is available according to\cite{mishra2020estimates}.
\end{remark}
Let $\mathcal{B}_1$  and $\mathcal{B}_2$ denote a series of differential operators that act on the interface.
Physically, \eqref{eq:231120-01} and \eqref{eq:231120-02}  naturally meet an interface condition 
\begin{equation}\label{eq:231120-03}
\mathcal{B}_1[\boldsymbol{u}_1(\boldsymbol{u}_I, t);\lambda_1]=\mathcal{B}_2[\boldsymbol{u}_2(\boldsymbol{u}_I, t);\lambda_2]
\end{equation}
to allow for information to communicate across the interface between $\Omega_1$ and $\Omega_2$.
Two residual functions are defined for \eqref{eq:231120-6} as the following:\begin{subequations}\label{eq:231120-04}
\begin{align}
& \boldsymbol{f}_1(\boldsymbol{x}, t;\lambda_1)=\boldsymbol{u}_{1t}(\boldsymbol{x}, t)+\mathcal{N}[\boldsymbol{u}_1(\boldsymbol{x}, t);{\lambda}_{1}]-\boldsymbol{g}_1(\boldsymbol{x}, t)\label{eq:231120-04-01},\\
&\boldsymbol{f}_2(\boldsymbol{x}, t;\lambda_2)=\boldsymbol{u}_{2t}(\boldsymbol{x}, t)+\mathcal{N}[\boldsymbol{u}_2(\boldsymbol{x}, t);{\lambda}_{2}]-\boldsymbol{g}_2(\boldsymbol{x}, t).\label{eq:231120-04-02}
\end{align}
\end{subequations}
A residual function is defined for \eqref{eq:231120-03}
 as the following:
 \begin{equation}\label{eq:231120-05}
\boldsymbol{f}_I(\boldsymbol{x}, t;\lambda_1,\lambda_2)=\mathcal{B}_1[\boldsymbol{u}_1(\boldsymbol{u}_I, t);\lambda_1]-\mathcal{B}_2[\boldsymbol{u}_2(\boldsymbol{u}_I, t);\lambda_2].
\end{equation}
Consequently, $ \boldsymbol{f}$ is defined as the following
\begin{equation}\label{eq:231212-12}
  \boldsymbol{f}= \boldsymbol{f}_1+ \boldsymbol{f}_2+ \boldsymbol{f}_I.\end{equation}
\subsection{PISAL for Solving PDEs with Heterogeneous Media}
A data-physics-hybrid method, PISAL, is first proposed to synchronously learn the fields corresponding to governing equations \eqref{eq:231120-01} and \eqref{eq:231120-02}.
Furthermore, the unavailable time-vary interface can be adaptively identified.
The proposed PISAL includes three neural networks:
1) $Net_I$ is for learning the time-varying interface $I$ satisfying \eqref{eq:231120-03};
2) $Net_1$ is for learning the  solution satisfying \eqref{eq:231120-01};
3) $Net_2$ is for learning the solution satisfying \eqref{eq:231120-02}.

Available hard sensors on $\bar\Omega=\Omega\cup \partial \Omega$ offer training dataset $D$, which is divided into $D = D_{1}\cup D_{2}$ with $D_{1}\cap D_{2}=D_{I}$.
$D_{1}$ and $D_{2}$ denote the training dataset sampled from $\Omega_1$ and $\Omega_2$, respectively;
$D_{I}$ denotes the training dataset sampled from the time-varying interface $I$. 
Note that the collocation point set $E$ does not necessarily correspond to $D$.
Then, $E=E_1\cup E_2$ is randomly sampled from $\Omega$ with  $E_1 \cap E_2=E_I$. $E_{1}$ and $E_{2}$ denote the collocation point sets sampled from $\Omega_1$ and $\Omega_2$, respectively;
$E_I$ denotes the collocation point set sampled from  $I$.

Consequently, PISAL  can be  trained by minimizing the loss function
\begin{equation} \label{eq:23111906}
{{\rm {MSE}}_M} = {\rm {MSE}}_{DM}+ {\rm {MSE}}_{PM}.
\end{equation}
${\rm MSE}_{DM}$ and  ${\rm MSE}_{PM}$ are the data-driven loss and physics-informed loss for training PISAL, respectively.
${\rm MSE}_{DM}$ is formulated as the following
\begin{equation}
\begin{aligned}\label{eq:23111907}
{\rm MSE}_{DM}&={\rm MSE}_{DM_U}+{\rm MSE}_{DM_I}.
\end{aligned}
\end{equation}
${\rm MSE}_{DM_{U}}$  and  ${\rm MSE}_{DM_I}$ are the data-driven loss for  training $Net_i$ ($i$=1,2) and  $Net_I$, respectively. They are formulated as the following:
\begin{equation}\label{eq:24011901}
\begin{aligned}
\hspace{-2mm}{\rm MSE}_{DM_U}&={\rm MSE}_{DM_{U1}}+{\rm MSE}_{DM_{U2}}\\
&=\displaystyle \frac{1}{{\rm card}\left(D_1\right)}\underset{{(\boldsymbol{x}, t,\boldsymbol{u})\in D_1}}{\sum} \hspace{-3mm}\left( \hat {\boldsymbol{u}}_1\left(\boldsymbol{x},t;\boldsymbol{\Theta}_{U_1}\right) \hspace{-0.5mm}-\hspace{-0.5mm}{\boldsymbol{u}_1\left(\boldsymbol{x},t\right)}\right)^{2}\\
&\hspace{2mm}+\displaystyle \frac{1}{{\rm card}\left(D_2\right)}\underset{{(\boldsymbol{x}, t,\boldsymbol{u})\in D_2}}{\sum} \hspace{-3mm}\left( \hat {\boldsymbol{u}}_2\left(\boldsymbol{x},t;\boldsymbol{\Theta}_{U_2}\right) \hspace{-0.5mm}-\hspace{-0.5mm}{\boldsymbol{u}_2\left(\boldsymbol{x},t\right)}\right)^{2},
\end{aligned}
\end{equation}
\vspace{-1mm}
\begin{equation}
{\rm MSE}_{DM_I}\hspace{-0.5mm}=\hspace{-0.5mm}\displaystyle \frac{1}{{\rm card}\left(D_I\right)}\underset{{(\boldsymbol{x}, t,{\boldsymbol{u}})\in D_I}}{\sum} \hspace{-2mm}\left( \hat {\boldsymbol{u}}_I\left(\boldsymbol{x},t;\boldsymbol{\Theta}_{U_I}\right) \hspace{-0.5mm}-\hspace{-0.5mm}{\boldsymbol{u}_I\left(\boldsymbol{x},t\right)}\right)^{2}.
\end{equation}
Here, $\hat {\boldsymbol{u}}_1\left(\boldsymbol{x}, t;\boldsymbol{\Theta}_{U_1}\right)$ is the function of $Net_1$, with $\boldsymbol{\Theta}_{U_1}$ being  a set of parameters; 
$\hat {\boldsymbol{u}}_2\left(\boldsymbol{x}, t;\boldsymbol{\Theta}_{U_2}\right)$ is the function of $Net_2$, with $\boldsymbol{\Theta}_{U_2}$ being  a set of parameters; and $\hat {\boldsymbol{u}}_I\left(\boldsymbol{x}, t;\boldsymbol{\Theta}_{U_I}\right)$ is the function of $Net_I$, with $\boldsymbol{\Theta}_{U_I}$ being  a set of parameters.
The differentials involved in the loss functions are numerically computed by using automatic differentiation \cite{baydin2018automatic}.
${\rm {MSE}}_{PM}$ is formulated as the following
\begin{equation}\label{eq:20221111-qp1}
\begin{aligned}
{\rm MSE}_{PM}&= {\rm MSE}_{PM_U}+{\rm MSE}_{PM_I},
\end{aligned}
\end{equation}
which is used to
regularize $\hat{ \boldsymbol{u}}_i\left(\boldsymbol{x},t;\boldsymbol{\Theta}_{U_i}\right)$ ($i$=1,2) and $\hat{ \boldsymbol{u}}_I\left(\boldsymbol{x},t;\boldsymbol{\Theta}_{U_I}\right)$ to satisfy \eqref{eq:231120-6} and \eqref{eq:231120-03}, respectively.
${\rm MSE}_{PM_U}$ and ${\rm MSE}_{PM_I}$ are physics-informed loss for training $Net_i$ ($i$=1,2) and $Net_I$, respectively. 
They can be obtained as the following 
\begin{equation}\label{eq:24011902}
\begin{aligned}
{\rm MSE}_{PM_U}&={\rm MSE}_{PM_{U1}}+{\rm MSE}_{PM_{U2}}\\
&=\displaystyle \frac{1}{{\rm card}\left(E_1\right)} \underset{(\boldsymbol{x}, t)\in {E_1}}{\sum}\hat {\boldsymbol{f}}_1\left(\boldsymbol{x}, t;\boldsymbol{\Theta}_{U_1};\lambda_1\right)^2\\
 &\hspace{2mm}+\displaystyle \frac{1}{{\rm card}\left(E_2\right)} \underset{(\boldsymbol{x}, t)\in {E_2}}{\sum}\hat {{\boldsymbol{f}}}_2\left(\boldsymbol{x}, t;\boldsymbol{\Theta}_{U_2};\lambda_2\right)^2,
 \end{aligned}
\end{equation}
\vspace{-1mm}
 \begin{equation}
 {\rm MSE}_{PM_I}=\displaystyle \frac{1}{{\rm card}\left(E_I\right)} \underset{(\boldsymbol{x}, t)\in {E_I}}{\sum}\hat {{\boldsymbol{f}}}_I\left(\boldsymbol{x}, t;\boldsymbol{\Theta}_{U_I};\lambda_1,\lambda_2\right)^2.
\end{equation}
Here, $\hat{{\boldsymbol{f}}}_1\left(\boldsymbol{x}, t;\boldsymbol{\Theta}_{U_1};\lambda_1\right)$ denotes a physics-informed approximation error for $Net_1$, an estimation to the residual function ${\boldsymbol{f}_1}(\boldsymbol{x},t;\lambda_1)$ \eqref{eq:231120-04-01} based on $\hat{\boldsymbol{u}}_1\left(\boldsymbol{x},t;\boldsymbol{\Theta}_{U_1}\right)$. $\hat{{\boldsymbol{f}}}_1\left(\boldsymbol{x}, t;\boldsymbol{\Theta}_{U_1};\lambda_1\right)$  is obtained as the following
\begin{equation}\notag
\begin{aligned}
&\hat{{\boldsymbol{f}}}_1\left(\boldsymbol{x}, t;\boldsymbol{\Theta}_{U_1};\lambda_1\right)\\
&\hspace{6mm}=\hat{\boldsymbol{u}}_{1t}(\boldsymbol{x},t;\boldsymbol{\Theta}_{U_1})+\mathcal{N}[\hat {\boldsymbol{u}}_1\left(\boldsymbol{x},t;\boldsymbol{\Theta}_{U_1}\right);{\lambda}_{1}]-\boldsymbol{g}_1(\boldsymbol{x}, t);
 \end{aligned}
\end{equation}
{${\boldsymbol{f}}_2\left(\boldsymbol{x}, t;\boldsymbol{\Theta}_{U_2};\lambda_2\right)$ 
denotes a physics-informed approximation error for $Net_2$, an estimation to the residual function ${\boldsymbol{f}_2}(\boldsymbol{x},t;\lambda_2)$ \eqref{eq:231120-04-02} based on $\hat{\boldsymbol{u}}_2\left(\boldsymbol{x},t;\boldsymbol{\Theta}_{U_2}\right)$. $\hat{{\boldsymbol{f}}}_2\left(\boldsymbol{x}, t;\boldsymbol{\Theta}_{U_2};\lambda_2\right)$  is obtained as the following
\begin{equation}\notag
\begin{aligned}
& \hat{{\boldsymbol{f}}}_2\left(\boldsymbol{x}, t;\boldsymbol{\Theta}_{U_2};\lambda_2\right)\\
&\hspace{6mm}=\hat{\boldsymbol{u}}_{2t}(\boldsymbol{x},t;\boldsymbol{\Theta}_{U_2})+\mathcal{N}[\hat {\boldsymbol{u}}_2\left(\boldsymbol{x},t;\boldsymbol{\Theta}_{U_2}\right);{\lambda}_{2}]-\boldsymbol{g}_2(\boldsymbol{x}, t);
 \end{aligned}
\end{equation}}{$\hat {{\boldsymbol{f}}}_I\left(\boldsymbol{x}, t;\boldsymbol{\Theta}_{U_I};\lambda_1,\lambda_2\right)$}
denotes a physics-informed approximation error for $Net_I$, an estimation to the residual function ${\boldsymbol{f}_I}(\boldsymbol{x},t;\lambda_1,\lambda_2)$ \eqref{eq:231120-05} based on
$\hat {\boldsymbol{u}}\left(\boldsymbol{x},t;\boldsymbol{\Theta}_{U}\right)$. 
Here $\hat {\boldsymbol{u}}\left(\boldsymbol{x},t;\boldsymbol{\Theta}_{U}\right)=\left\{ \hat{\boldsymbol{u}}_1\left(\boldsymbol{x},t;\boldsymbol{\Theta}_{U_1}\right),\hat{\boldsymbol{u}}_2\left(\boldsymbol{x},t;\boldsymbol{\Theta}_{U_2}\right),\hat{\boldsymbol{u}}_I(\boldsymbol{x},t;\boldsymbol{\Theta}_I) \right\}$.
$\hat {{\boldsymbol{f}}}_I\left(\boldsymbol{x}, t;\boldsymbol{\Theta}_{U_I};\lambda_1,\lambda_2\right)$  is obtained as the following
\begin{equation}\notag
\begin{aligned}
& \hat {{\boldsymbol{f}}}_I\left(\boldsymbol{x}, t;\boldsymbol{\Theta}_{U_I};\lambda_1,\lambda_2\right)\\
&\hspace{12mm}= \mathcal{B}_1[\hat{\boldsymbol{u}}_1 (\hat{\boldsymbol{u}}_I(\boldsymbol{x},t;\boldsymbol{\Theta}_I),  t, \boldsymbol{\Theta}_{U_1} ); \lambda_1] \\
&\hspace{12mm}-\mathcal{B}_2[\hat{\boldsymbol{u}}_2 (\hat{\boldsymbol{u}}_I(\boldsymbol{x},t; \boldsymbol{\Theta}_I), t,  \boldsymbol{\Theta}_{U_2} ) ; \lambda_2 ].
\end{aligned}
\end{equation}
Let 
\begin{equation}\label{eq:24010802}
\begin{aligned}
\hat{\boldsymbol{e}}_{M}\left(\boldsymbol{x},t;\boldsymbol{\Theta}_{U}\right)
&= \hat {\boldsymbol{u}}\left(\boldsymbol{x},t;\boldsymbol{\Theta}_{U}\right) -{\boldsymbol{u}\left(\boldsymbol{x},t\right)}\\
&=\hat {\boldsymbol{u}}_1\left(\boldsymbol{x},t;\boldsymbol{\Theta}_{U_1}\right) -{\boldsymbol{u}_1\left(\boldsymbol{x},t\right)}\\&\hspace{1mm}+
\hat {\boldsymbol{u}}_2\left(\boldsymbol{x},t;\boldsymbol{\Theta}_{U_2}\right) -{\boldsymbol{u}_2\left(\boldsymbol{x},t\right)}\\
&\hspace{1mm}+\hat {\boldsymbol{u}}_I\left(\boldsymbol{x},t;\boldsymbol{\Theta}_{U_I}\right) -{\boldsymbol{u}_I\left(\boldsymbol{x},t\right)}\end{aligned},
\end{equation}
denote a data approximation error for PISAL.
Let 
\begin{equation}
\begin{aligned}\label{eq:240108}
 &\hat {\boldsymbol{f}}\left(\boldsymbol{x}, t;\boldsymbol{\Theta}_{U};\lambda_1,\lambda_2\right)\\
&= \boldsymbol{\hat u}_{1t}\left(\boldsymbol{x}, t;\boldsymbol{\Theta}_{U_1}\right)+\mathcal{N}[\boldsymbol{\hat u}_1(\boldsymbol{x}, t;\boldsymbol{\Theta}_{U_1});{\lambda}_{1}]-\boldsymbol{g}_1\left(\boldsymbol{x}, t\right)
\\
&+ \boldsymbol{\hat u}_{2t}\left(\boldsymbol{x}, t;\boldsymbol{\Theta}_{U_2}\right)+\mathcal{N}[\boldsymbol{\hat u}_2(\boldsymbol{x}, t;\boldsymbol{\Theta}_{U_2});{\lambda}_{2}]-\boldsymbol{g}_2\left(\boldsymbol{x}, t\right) \\
&+\mathcal{B}_1[\hat{\boldsymbol{u}}_1 (\hat{\boldsymbol{u}}_I(\boldsymbol{x},t; \boldsymbol{\Theta}_I),  t, \boldsymbol{\Theta}_{U_1});\lambda_1]\\
&-\mathcal{B}_2[\hat{\boldsymbol{u}}_2 (\hat{\boldsymbol{u}}_I(\boldsymbol{x},t; \boldsymbol{\Theta}_I), t, \boldsymbol{\Theta}_{U_2} ) ; \lambda_2 ].
\end{aligned}
\end{equation}
denote a physics-informed approximation error for PISAL, an estimation to $\boldsymbol{f}\left(\boldsymbol{x}, t;\lambda_1, \lambda_2\right)$ \eqref{eq:231212-12} based on $\hat {\boldsymbol{u}}\left(\boldsymbol{x},t;\boldsymbol{\Theta}_{U}\right)$.
In fact, the approximation $\hat {\boldsymbol{u}}=\left\{\hat {\boldsymbol{u}}_1,\hat {\boldsymbol{u}}_2,\hat {\boldsymbol{u}}_I\right\}$ can be obtained by  perturbing  \eqref{eq:231120-6} and \eqref{eq:231120-03} with $\hat{\boldsymbol{e}}_{M}$ \eqref{eq:24010802} and $\hat {\boldsymbol{f}}$ \eqref{eq:240108}.
\vspace{-2mm}
\subsection{SAL Strategy for Each Subdomain}
A decomposition criterion is defined to adaptively distinguish the attributions of the measurements and collocation points.
To further explicitly formulate the criterion,  let $\mathcal{S}=\mathcal{S}_1\cup \mathcal{S}_1$ with $\mathcal{S}_1$=$\left\{v\in\mathcal{S}_1| \left(x_v,t\right)\leq I \right\}$, $\mathcal{S}_2$=$\left\{v\in\mathcal{S}_2| \left(x_v,t\right)\geq I \right\}$, and $\left\{v\in\mathcal{S}_1\cap \mathcal{S}_2| \left(x_v,t\right)=I \right\}$.
The criterion is as follows:
\begin{equation} \label{eq:23111906}
\left\{
\begin{aligned}
&(\boldsymbol{x},t)\in \Omega_1\times [0, T],\  \lambda_i=\lambda_1, \ \forall v\in\mathcal{S}_1 \  \\
&(\boldsymbol{x},t)\in \Omega_2\times [0, T],\ \lambda_i=\lambda_2, \ \forall v\in\mathcal{S}_2
\end{aligned}.
\right.
\end{equation}Accordingly, SAL strategy is proposed to decompose and optimize each subdomain to iteratively estimate $\boldsymbol{\Theta}_1$, $\boldsymbol{\Theta}_2$, and  $\boldsymbol{\Theta}_{U_I}$.
Let $k$ denote the iternative step.
Suppose $\left\{\hat{\boldsymbol{\Theta}}_{U_1}^{(k)}, \hat{\lambda}_1^{(k)}\right\}$ 
and $\left\{\hat{\boldsymbol{\Theta}}_{U_2}^{(k)}, \hat{\lambda}_{2}^{(k)}\right\}$ are the iterations of  $Net_1$ and $Net
_2$ at $k^{\rm{th}}$ iterative step, respectively.
$\boldsymbol{\hat\Theta}_{U_I}^{(k+1)}$, $\boldsymbol{\hat\Theta}_{U_1}^{(k+1)}$, $\hat{\lambda}_1^{(k+1)}$,  $\boldsymbol{\hat\Theta}^{(k+1)}_{U_2}$, and $\hat{\lambda}_2^{(k+1)}$   can be obtained by solving the following optimization problem
\begin{equation}\label{eq:24010901}
\begin{aligned}    
\displaystyle
&\boldsymbol{\hat\Theta}_{U_I}^{(k+1)}
=\underset{\boldsymbol{\Theta}_{U_I}}{\arg \min }
\left\{{{\rm MSE}_{DM_I}\left(\boldsymbol{\Theta}_{U_I};\hat{\boldsymbol{\Theta}}_{U_1}^{(k)},\hat{\boldsymbol{\Theta}}_{U_2}^{(k)}  \right)}\right.\\
&\hspace{12mm}\left.+{{\rm MSE}_{PM_I}\left(\boldsymbol{\Theta}_{U_I};\hat{\boldsymbol{\Theta}}_{U_1}^{(k)},\hat{\lambda}_1^{(k)},\hat{\boldsymbol{\Theta}}_{U_2}^{(k)},\hat{\lambda}_2^{(k)}\right)}\right\}
\end{aligned}
\end{equation}
and 
\begin{equation}\label{eq:24010902}
\begin{aligned}    
\displaystyle
&\left\{\boldsymbol{\hat\Theta}_{U_1}^{(k+1)},\hat{\lambda}_1^{(k+1)},\boldsymbol{\hat\Theta}_{U_2}^{(k+1)},\hat{\lambda}_2^{(k+1)}\right\}\\
&=\underset{\left\{{\boldsymbol{\Theta}}_{U_1},\lambda_1,\boldsymbol{\Theta}_{U_2},\lambda_2\right\}}{\arg \min }
\left\{{{\rm MSE}_{DM_U}\left(\boldsymbol{\Theta}_{U_1},\boldsymbol{\Theta}_{U_2};\hat{\boldsymbol{\Theta}}_{U_I}^{(k+1)} \right)}\right.\\&\hspace{5mm}+\left.{{\rm MSE}_{PM_U}\left({\boldsymbol{\Theta}}_{U_1},\lambda_1,\boldsymbol{\Theta}_{U_2},\lambda_2;\hat{\boldsymbol{\Theta}}_{U_I}^{(k+1)}\right)}\right\}
\end{aligned}.
\end{equation}
The unknown PDE parameters $\lambda_i$ for $i$ = 1, 2 can be concurrently identified by solving the optimization  problems \eqref{eq:24010901} and \eqref{eq:24010902}, which can be referred to solving PDE inverse problems.
Let 
\begin{subequations}\label{eq2401905}
\begin{align}
{\rm MSE}_{Ui}&={\rm MSE}_{DM_{Ui}}+{\rm MSE}_{PM_{Ui}}\label{eq24011903}, {\quad i=1,2}\\
&{\rm MSE}_I={\rm MSE}_{DM_I}+{\rm MSE}_{PM_I}\label{eq24011904},
\end{align}
\end{subequations}
where ${\rm MSE}_{DM_{Ui}}$ and ${\rm MSE}_{PM_{Ui}}$ have been defined by \eqref{eq:24011901} and \eqref{eq:24011902}, respectively. Consequently, the details of the SAL strategy are described in Algorithm \ref{alg:algorithm-label}.
The framework of PISAL for industrial system modeling in hetergerous media is shown in Fig.\ref{fig:PISAL}.
\begin{algorithm}
    \caption{{The SAL strategy to decompose and optimize each subdomain.}}
    \label{alg:algorithm-label}
    \begin{algorithmic}
   \STATE  \textbf{Initialize} ($k=0$)  the maximum iterations $k_{max}$ and threshold $\delta_{Train}$. \\
   1. Input spatial and temporal variables $\boldsymbol{x}$ and $t$.\\
   2. Randomly sample the training dataset $(\boldsymbol{x}, t, u)\in D$  and collocation points $(\boldsymbol{x}, t)\in E$ from the whole $\Omega$. 
   
3. Randomly  initialize $Net_1$ and $Net_2$ by using ${\boldsymbol{\Theta}}^{(0)}_{U_1}$ and ${\boldsymbol{\Theta}}^{(0)}_{U_2}$, respectively.
Meanwhile, randomly initialize $Net_I$ by using  ${\boldsymbol{\Theta}}^{(0)}_{U_I}$.
  \WHILE {$k\leq k_{max}$ or ${\rm{MSE}}_M^{k}\leq\delta_{Train}$}
   \STATE 4.  Train $Net_I$ based on $(\boldsymbol{x},t,u)\in D_I^{(k)}$ and $(\boldsymbol{x},t)\in E_I^{(k)}$ to obtain $\boldsymbol{\hat\Theta}_{U_I}^{(k+1)}$ by solving the optimization problem \eqref{eq:24010901}. Based on the iteration result $\boldsymbol{\hat\Theta}_{U_I}^{(k+1)}$ and decomposition criterion \eqref{eq:23111906}, divide training dataset $D$ and collocation points $E$ into $D_i^{(k+1)}$ and $E_i^{(k+1)}$, respectively.
 \IF{$D_I^{(k+1)}\neq \varnothing $ and $E_I^{(k+1)}\neq \varnothing$}
   \STATE {5.  Obtain $D_I^{(k+1)}$ and $E_I^{(k+1)}$ .}
\ELSE{\STATE 6. Return step 2.}
\ENDIF
    \STATE 7.   
Train $Net_1$ and  $Net_2$  by solving the optimization problem \eqref{eq:24010902} to obtain  $\left\{\boldsymbol{\hat\Theta}_{U_1}^{(k+1)},\hat{\lambda}_1^{(k+1)}\right\}$ and $\left\{\boldsymbol{\hat\Theta}_{U_2}^{(k+1)},\hat{\lambda}_2^{(k+1)}\right\}$, respectively.
Train $Net_1$ based on $(\boldsymbol{x},t,u)\in D_1^{(k+1)}$, $(\boldsymbol{x},t)\in E_1^{(k+1)}$, $(\boldsymbol{x},t,u) \in D_I^{(k+1)}$, and $(\boldsymbol{x},t) \in E_I^{(k+1)}$;
    Train $Net_2$ based on $(\boldsymbol{x},t,u)\in D_2^{(k+1)}$, $(\boldsymbol{x},t)\in E_2^{(k+1)}$,  $(\boldsymbol{x},t,u)\in D_I^{(k+1)}$, and $(\boldsymbol{x},t)\in E_I^{(k+1)}$.
  \STATE 8. Update $\boldsymbol{\Theta}_{U_i}^{(k+1)}$ and $\lambda_i^{(k+1)}$  according to ${\rm MSE}^{(k+1)}_{Ui}$ \eqref{eq24011903}. Update $\boldsymbol{\Theta}_I^{(k+1)}$ according to  ${\rm MSE}^{(k+1)}_{I}$ \eqref{eq24011904}.
\STATE 9. $k=k+1$.
         \ENDWHILE
    \STATE \textbf{Return} the output $\hat {\boldsymbol{u}}_1(\boldsymbol{x},t;\boldsymbol{\hat\Theta}_{U_1})$, $\hat {\boldsymbol{u}}_2(\boldsymbol{x},t;\boldsymbol{\hat\Theta}_{U_2})$, $\hat{\lambda}_1$, $\hat{\lambda}_2$, and $\hat {\boldsymbol{u}}_I(\boldsymbol{x},t;\boldsymbol{\hat\Theta}_{U_I})$.
    \end{algorithmic}
\end{algorithm} 
\begin{figure*}[hpt]
\begin{center}
\includegraphics[width=15cm]{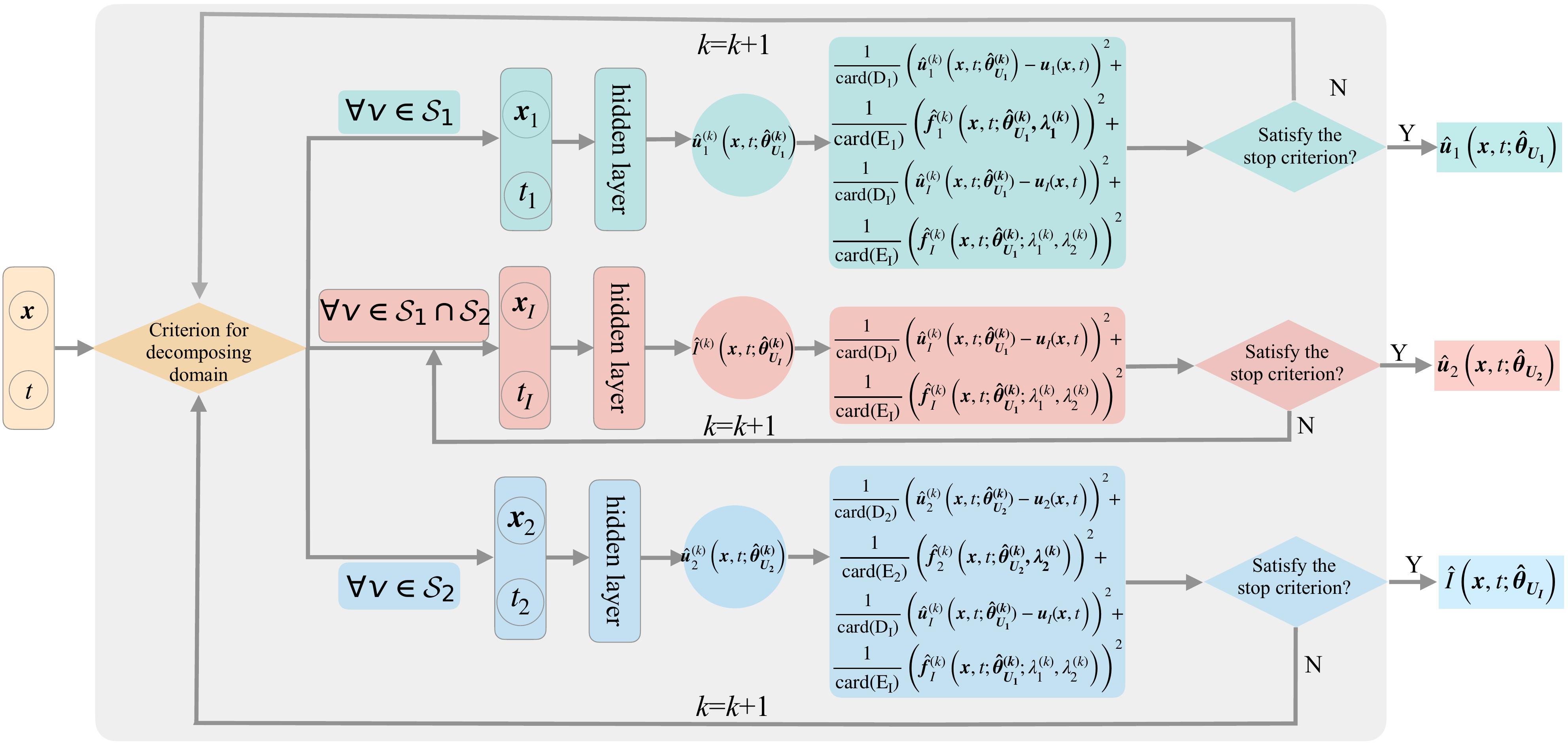}
\caption{Framework of PISAL for industrial system modeling in hetergerous media.}
\label{fig:PISAL}
\end{center}
\end{figure*}
{\subsection{Approximation Theorem for PISAL}\label{Subsection II-D}
This section concisely establishes the mathematical justification for the proposed PISAL.
To present the theoretical framework on the approximation error $\Vert \hat e_M \Vert$ \eqref{eq:24010802}, we first assume $\boldsymbol{u}_i$ satisfies the following well-posed condition  based on Remark \ref{Rem:24011202}.
\newtheorem{assumption}{Assumption}
\begin{assumption}\label{Assum:23121602}
For each function   $\boldsymbol{g}_{\boldsymbol{u}i}$ and $\boldsymbol{g}_{\boldsymbol{v}i}$ satisfying \eqref{eq:231120-6}, the corresponding  unique solutions $\boldsymbol{u}_i, \boldsymbol{v}_i \in X_i$ satisfy
\begin{equation}\label{eq:231217}
\begin{aligned}
\left\| {\boldsymbol{u}_{i}-\boldsymbol{v}_{i}}\right\|  \leq  & C \|\boldsymbol{g}_{\boldsymbol{u}i}-\boldsymbol{g}_{\boldsymbol{v}i}\| 
\end{aligned}
\end{equation}
 with a fixed and finite constant $0<C\in \mathbb{R}$ when 
the measurements sampled from any subset $\mathbb{D}_i^{'}\subset \mathbb{D}_i$.
Here,  $C$ can be referred to as the Lipschitz constant of PDEs \eqref{eq:231120-6}; $\|\cdot\|$ denotes the  $L^1$-norm.
\end{assumption}

To delve into the general approximation capability of the proposed PISAL, let the second power of $L^2$-norm $L_{DM}$ and $L_{PM}$ instead of MSE in \eqref{eq:23111907}
 and \eqref{eq:20221111-qp1}, respectively. 
Let $L_M$=$L_{DM}+L_{PM}$.
Consequently, the following theorem guarantees the general approximation capability of PISAL.
\newtheorem{theorem}{\bf Theorem}
\begin{theorem}\label{thm1} 
If the Assumption \ref{Assum:23121602} holds for any $\mathbb{D}^{'}\subset \mathbb{D}$.
Then, for all $\varepsilon>0$, there exists a $\delta>0$,
\begin{equation}
L_{M}<\delta \Longrightarrow \left\Vert\boldsymbol{\hat u}-\boldsymbol{u}\right\Vert<\varepsilon
\end{equation}
\end{theorem}
\begin{proof}
According to Assumption \ref{Assum:23121602},
let 
$\boldsymbol{g}_{\boldsymbol{u}1}=\boldsymbol{\hat u}_{1t}\left(\boldsymbol{x}, t;\boldsymbol{\Theta}_{U_1}\right)+\mathcal{N}[\boldsymbol{\hat u}_1(\boldsymbol{x}, t;\boldsymbol{\Theta}_{U_1});{\lambda}_{1}]-\boldsymbol{\hat f}_{1}\left(\boldsymbol{x}, t;\boldsymbol{\Theta}_{U_1}\right)$ and $\boldsymbol{g}_{\boldsymbol{v}1}=\boldsymbol{g}_1$. Meanwhile,
$\boldsymbol{g}_{\boldsymbol{u}2}$, $\boldsymbol{g}_{\boldsymbol{v}2}$,  $\boldsymbol{g}_{\boldsymbol{u}I}$, and $\boldsymbol{g}_{\boldsymbol{v}I}$ can be similarly obtained.
Based on our previous work\cite{wang2023coupled},  there exists a finite Lipschitz constant $0<C\in \mathbb{R}$ satisfying
\begin{equation}\notag\label{23121601}
\begin{aligned}
\left\|\hat{\boldsymbol{u}}\left(\boldsymbol{x},t;\boldsymbol{\Theta}_{U}\right)-\boldsymbol{u}\left(\boldsymbol{x},t\right)\right\|&= \left\| \hat{\boldsymbol{e}}_M\right\|\leqslant C\left\|\hat{\boldsymbol{f}}\right\|\\
 &=C \left(T{\rm card}\left(\bar\Omega\right)\right)^{\frac{1}{2}}L_{M}^{\frac{1}{2}}\\
 &=C_{GE}L_M^{\frac{1}{2}}.
\end{aligned}
\end{equation}
with constants $C_{GE}=C\left(T{\rm card}\left(\bar\Omega\right)\right)^{\frac{1}{2}}$.
Consequently, let 
$$\delta=\frac{\varepsilon^{2}}{C_{GE}^{2}}
$$ 
for which $L_{M} < \delta$ yields 
$$\left\Vert\boldsymbol{\hat u}-\boldsymbol{u}\right\Vert \leq C_{GE} \delta^{\frac{1}{2}}=\varepsilon$$
\end{proof}

\begin{remark}\label{231217}
Let $\boldsymbol{\Theta}_{U_i}$ and $\boldsymbol{\Theta}_{U_I}$ denote the parameters of $\hat{\boldsymbol{u}}_i\left(x, t ; \boldsymbol{\Theta}_{U_i}\right)$ and $\hat{\boldsymbol{u}}_I\left(x, t ; \boldsymbol{\Theta}_{U_I}\right)$ for $i$ = 1, 2, respectively. Theorem  \ref{thm1}  indicates that there is a series of functions $\left\{\hat{\boldsymbol{u}}_1\left(x, t ; \boldsymbol{\Theta}_{U_i}^{(j)}\right),\hat{\boldsymbol{u}}_2\left(x, t ; \boldsymbol{\Theta}_{U_2}^{(j)}\right),\hat{\boldsymbol{u}}_I\left(x, t ; \boldsymbol{\Theta}_{U_I}^{(j)}\right)\right\}$ to make loss function $L_{M}\left(\boldsymbol{\Theta}_{U_1}^{(j)},\boldsymbol{\Theta}_{U_2}^{(j)}, \boldsymbol{\Theta}_{U_I}^{(j)}\right) \rightarrow 0 \ (j \rightarrow \infty)$. That is, there exits a large $J$, $L_{M}\left(\boldsymbol{\Theta}_{U_1}^{(j)}, \boldsymbol{\Theta}_{U_2}^{(j)}, \boldsymbol{\Theta}_{U_I}^{(j)}\right)$ will be tiny for $j \geq J$. Consequently,  Theorem  \ref{thm1} theoretically ensures Algorithm \ref{alg:algorithm-label} is effective, in which steps 4, 7, and 8 iteratively reduce the loss function.
\end{remark}
\section{Illusive Example}\label{Illusive Example}
In this section, 
we validate the performance of the proposed PISAL by using the classical two-phase Stefan problem and mixed Navier-Stokes problem. 
We employ FCNNs initialized with Xavier, utilizing hyperbolic tangent activation functions.
The optimization of parameters $\boldsymbol{\Theta}_{U_1}$, $\boldsymbol{\Theta}_{U_2}$, and $\boldsymbol{\Theta}_{U_I}$ is carried out by using 
L-BFGS, while
Adam is employed for optimizing $\lambda_1$ and $\lambda_2$.
Note that the abovementioned setups can ensure Algorithm  \ref{alg:algorithm-label} works, as mentioned in Remark \ref{231217}.
In this section, all implementations are executed within the PyTorch 3.9.0 open-source deep learning framework.
The computations are performed on a computer equipped with 32 GB memory and an NVIDIA GeForce RTX 3060 GPU.

We evaluate the prediction performance of  the proposed PISAL  by means of root mean squared error (RMSE)
\begin{equation}\notag
\rm {RMSE}=\sqrt{\frac{1}{{\rm card}\left(T_e\right)}\sum_{(\boldsymbol{x},t,\boldsymbol{u})\in T_e} \left({\hat {\boldsymbol{u}}\left(\boldsymbol{x},t\right)} -
{\boldsymbol{u}}\left(\boldsymbol{x}, t\right)\right)^{2}},
\end{equation}
where  $T_e$ is the testing dataset.
$\hat{\boldsymbol{u}}\left(\boldsymbol{x}, t\right)$ and ${\boldsymbol{u}}\left(\boldsymbol{x}, t\right)$ denote the prediction and the corresponding ground truth, respectively.
To further validate the prediction performance of PISAL, the Pearson correlation coefficient (CC)
\begin{equation}\notag
 {\rm CC} = \frac{{\rm Cov}\left(\hat {\boldsymbol{u}}\left(\boldsymbol{x},t\right), \boldsymbol{u}\left(\boldsymbol{x},t\right)\right)}{\sqrt{{\rm Var}   \left(\hat {\boldsymbol{u}}\left(\boldsymbol{x},t\right)\right)}\sqrt{{\rm Var} \left(\boldsymbol{u}\left(\boldsymbol{x}, t \right)\right)}}
\end{equation} 
is also used to measure the similarity between the prediction and the ground truth,
where ${\rm Cov}\left(\cdot,\cdot\right)$ is covariance and ${\rm Var}  (\cdot) $ is variance.
Meanwhile, the percentage error 
$$
{\rm {PE}}=\frac{{|\hat {\boldsymbol{u}}\left(\boldsymbol{x},t\right)} -
\boldsymbol{u}\left(\boldsymbol{x}, t\right)|}{\boldsymbol{u}\left(\boldsymbol{x}, t\right)}\times100\%
$$
is used to evaluate the  identification performance.  
\subsection{Two-Phase Stefan Problem}\label{section:Two-Phase Stefan Problem}
The following governing equation is first considered to generate data, 
\begin{equation}\label{eq:231120-07}
\frac{\partial u_i}{\partial t}=k_i \frac{\partial^2 u_i}{\partial x^2}, \quad x \in \Omega_i, \quad t \in[0,T], \quad i=1,2,
\end{equation}
which can be used to model the two-phase Stefan problems, such as continuum systems with phase transitions\cite{gomez2019review}, moving obstacles\cite{kim2021stefan}, multi-phase dynamics\cite{nandi2022second}, and competition for resources\cite{wang2021deep}.
We specifically use \eqref{eq:231120-07} to model the heat transfer in a two-phase.
 The temperature distributions $u_i(x,t)$ within each of the two phases satisfy the heat equation \eqref{eq:231120-07}.
Let $k_1=2$ and $k_2=1$ denote the thermal diffusivity parameters of $\Omega_1$ and $\Omega_2$, respectively.
For more details of setups, readers are referred to   \cite{cai2021physics}.
Physically, the following natural boundary conditions with respect to energy balance 
\begin{subequations}\label{eq:231120-08}
\begin{align}
& u_1(s(t), t)=u_2(s(t), t)=u^*,\label{eq:231120-08-1}\\
& s^{\prime}(t)=\alpha_1 \frac{\partial u_1}{\partial x}(s(t), t)+\alpha_2 \frac{\partial u_2}{\partial x}(s(t), t), \label{eq:231120-08-2}\\
& s(0)=s_0,\label{eq:231120-08-3}
\end{align}
\end{subequations}
can be obtained. 
$u^{*}$ in \eqref{eq:231120-08-1} is the temperature distribution on the interface.
$\alpha_1=-2$ and $\alpha_2=1$ in \eqref{eq:231120-08-2}; $s_0=1/2$ in \eqref{eq:231120-08-3}.
The  exact solutions for time-varying interface $s(t)$ and the temperature distributions   $u_1$ and $u_2$ are as follows: 
\begin{equation}
\left\{
\begin{aligned}
&u_1(x, t)=2(\exp ((t+1 / 2-x) / 2)-1) \\
&u_2(x, t)=\exp ((t+1 / 2-x)-1) \\
&s(t)=t+1 / 2
\end{aligned}
\right.,
\end{equation}
which are used to generate the training and testing datasets. Considering the practical situations, the exact thermal diffusivity parameters $k_i$ for $i=1,2$ and the time-varying interface $s(t)$ are assumed to be unknown.

In this case, $Net_1$, $Net_2$, and $Net_I$ consist of one hidden layer with 100 neurons.
A total number of training data $N_u=220$ is randomly sampled from the whole domain $\Omega$, including 20 training data sampled from the initial condition. 
Moreover,  the  collocation points $N_f=2000$ are randomly sampled to regularize the structure of \eqref{eq:231120-07} and \eqref{eq:231120-08}.
The abovementioned training dataset and the magnitude of predictions $\hat u_i(x, t)$ satisfying \eqref{eq:231120-07} are shown in Fig. \ref{fig:HeatEquation}(a). 
\begin{figure}[h]
\begin{center}
\includegraphics[width=7cm]{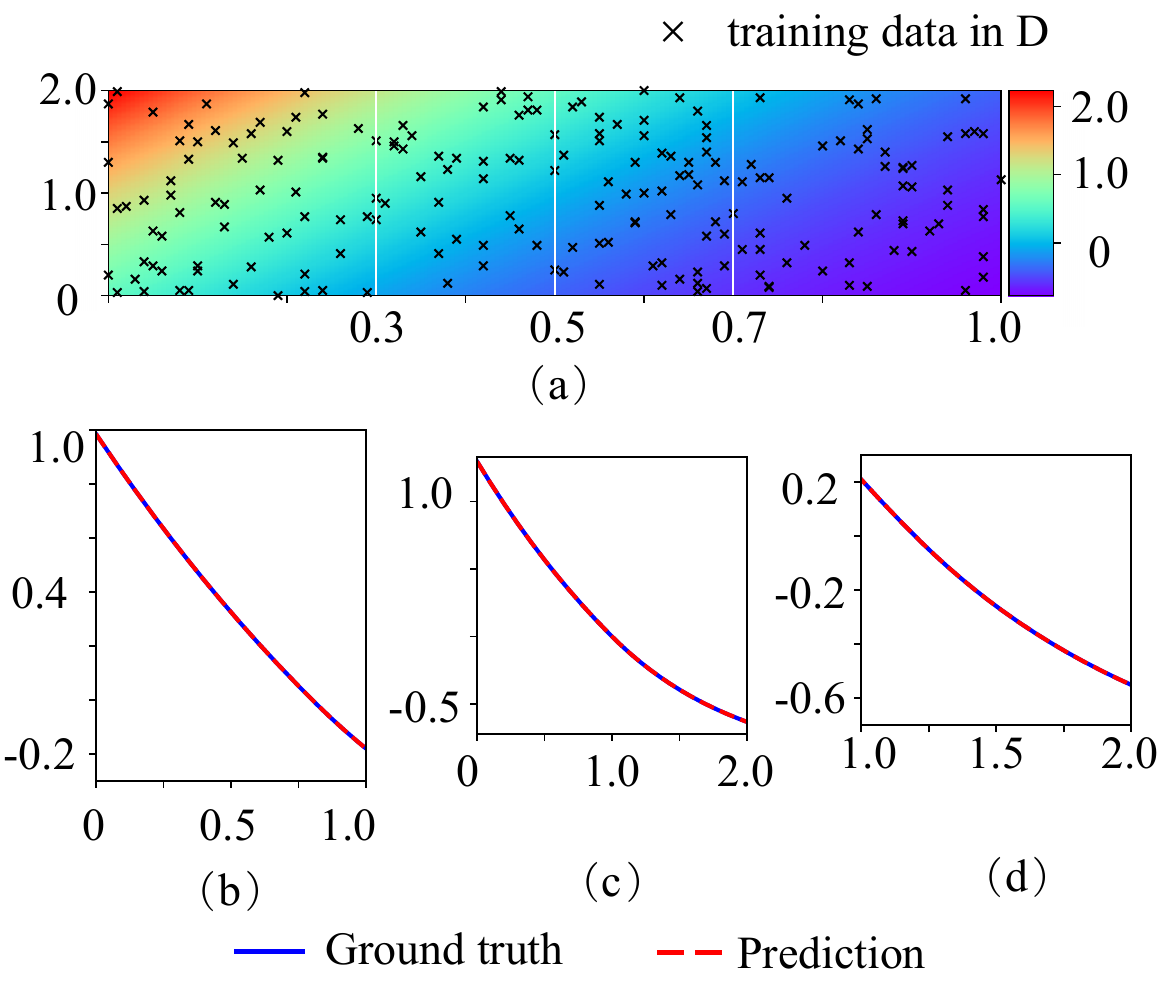}
\caption{(a) Predictions $\hat u_i\left(x,t\right)$ for \eqref{eq:231120-07}.
(b), (c), and (d) Comparisons of predictions and ground truths corresponding to fixed-time $t$=0.3, 0.5, and 0.7 snapshots depicted by dashed vertical lines in (a).}
\label{fig:HeatEquation}
\end{center}
\end{figure}
Table \ref{tb:HeatEquation_parameter} 
summarizes the results with respect to the identification performance and prediction performance.
\begin{table}[h]
\begin{center}
\caption{Comparisons for Stefan problem}
\label{tb:HeatEquation_parameter}
\setlength{\tabcolsep}{0.3mm}{
\begin{tabular}{cccccc}
\hline
\tabincell{c}{Method}&\tabincell{c}{Domain}&
\tabincell{c}{$k_{i}$} &
\tabincell{c}{PE(\%)} &
\tabincell{c}{${\rm RMSE}_{u_i}$} &
\tabincell{c}{${\rm RMSE}_{I}$} 
\\\hline
\multirow{2}{*}{\tabincell{c}{PISAL}}&$\Omega_1$&\textbf{1.997} &\textbf{0.257}& \textbf{7.261{\it e}-04}&\multirow{2}{*}{\tabincell{c}{\textbf{1.986{\it e}-02}}}\\
&$\Omega_2$&\textbf{0.997} & \textbf{0.327} & \textbf{9.268{\it e}-04}&  \\\hline
\multirow{2}{*}{\tabincell{c}{PINNs}}&$\Omega_1$&$2.151$ &40.460&2.224{\it e}-02&\multirow{2}{*}{\tabincell{c}{---}}\\
&$\Omega_2$&2.151 & 58.855 & 1.497{\it e}-02&  \\\hline
\multirow{2}{*}{\tabincell{c}{indicator-\\PINNs}}&$\Omega_1$&1.595 &15.646&1.340{\it e}-02&\multirow{2}{*}{3.700{\it e}-01}\\
&$\Omega_2$&1.588 & 115.646 & 1.943{\it e}-02&  \\\hline
\end{tabular}}
\end{center}
\end{table}
Moreover, we compare the predictions and ground truths at fixed-time $t$=0.3, 0.5, and 0.7  in Table \ref{tb:HeatEquation}, which correspond to Fig.\ref{fig:HeatEquation} (b), (c), and (d), respectively.
\begin{table}[h]
\begin{center}
\caption{Evaluation criteria with respect to temporal snapshots depicted by dashed vertical lines in Fig.~\ref{fig:HeatEquation}-(a) of different methods.}
\label{tb:HeatEquation}
\setlength{\tabcolsep}{0.4mm}{
\begin{tabular}{ccccccc}
\hline
\tabincell{c}{Method}&
\tabincell{c}{Criteria} &
\tabincell{c}{0.3} &
\tabincell{c}{0.5} &
\tabincell{c}{0.7} &
\tabincell{c}{$\left[0,2\right]\hspace{-0.6mm}\times\hspace{-0.6mm}\left[0,1\right]$}
\\\hline
\multirow{2}{*}{\tabincell{c}{PISAL}}&{\tabincell{c}{${\rm RMSE}$}}&  \bf{6.888{\it e}-04} & \bf{7.111{\it e}-04}& \bf{7.311{\it e}-04}&\bf{8.339{\it e}-04}  \\
&{\tabincell{c}{CC}} & \bf{9.999{\it e}-01} & \bf{9.999{\it e}-01}& \bf{9.999{\it e}-01} & \bf{9.999{\it e}-01}  \\\hline
\multirow{2}{*}{\tabincell{c}{PINNs}}&{\tabincell{c}{${\rm RMSE}$}}&1.308{\it e}-02&1.610{\it e}-02& 1.736{\it e}-02&1.679{\it e}-02\\
&{\tabincell{c}{CC}}&9.996{\it e}-01& 9.997{\it e}-01 & 9.996{\it e}-01&9.997{\it e}-01  \\\hline
\multirow{2}{*}{\tabincell{c}{indicator-\\PINNs}}&{\tabincell{c}{${\rm RMSE}$}}&1.903{\it e}-02&1.761{\it e}-02& 1.252{\it e}-02&1.669{\it e}-02\\
&{\tabincell{c}{CC}}&9.992{\it e}-01& 9.995{\it e}-01 & 9.998{\it e}-01&9.996{\it e}-01  \\\hline
\end{tabular}}
\end{center}
\end{table}To evaluate the performance, recent relevant methods,  PINNs\cite{raissi2019physics} and indicator-PINNs\cite{cai2021physics} are conducted for comparisons.
For fairness, the total number of training data and collocation points are kept fixed at $N_u=220$ and $N_f=2000$, respectively.
The best values are bold in Tables  \ref{tb:HeatEquation_parameter} and \ref{tb:HeatEquation} to clearly show the performance among different methods with identical setups. 
The performance of PINNs suggests that the jumping of PDE parameters through the media interface should be considered to allow for information to be communicated across the interface between $\Omega_1$ and $\Omega_2$.
The performance analysis of indicator-PINNs suggests that inaccuracies in model predictions are associated with certain issues intrinsic to the model, as discussed in \cite{cai2021physics}.

\subsection{Mixed Navier-Stokes Problem}
\label{section:5}
Furthermore, we consider the mixed Navier-Stokes problem with friction-type interface conditions to generate data. 
According to \cite{li2021schwarz}, the following 
stationary Stokes-Stokes coupling problem at each time step has to be solved
\begin{subequations}\label{eq:23112201}
\begin{align}
\hspace{-3.5mm}
\sigma \boldsymbol{u}_i\hspace{-0.5mm}-\hspace{-0.5mm}\nabla \cdot\left(\nu_i \nabla \boldsymbol{u}_i-p_i \mathbb{I}\right) & \hspace{-0.5mm}=\hspace{-0.5mm}\boldsymbol{g}_i, & & \text { in } \Omega_i, \\
\nabla \cdot \boldsymbol{u}_i & =0, & & \text { in } \Omega_i, \\ 
\hspace{-1.8mm}\left(\left(\nu_i \nabla \boldsymbol{u}_i\hspace{-0.5mm}-\hspace{-0.5mm}p_i \mathbb{I}\right) \cdot \boldsymbol{n}_i\right) \hspace{-0.5mm}\cdot\hspace{-0.5mm} \tau   &=\kappa\left(\boldsymbol{u}_j\hspace{-0.5mm}-\hspace{-0.5mm}\boldsymbol{u}_i\right) \cdot \tau, &  & \text { on } \Gamma, \\ 
\text {for }i=1,2&,\   j=3-i, \nonumber\\
\boldsymbol{u}_i \cdot \boldsymbol{n}_i & =0, & & \text { on } \Gamma, \\
\boldsymbol{u}_i & =0, &  & \text { on } \partial \Omega_i \backslash \Gamma,
\end{align}
\end{subequations}
where the vector field  $\boldsymbol{u}_i =[u_{i}^{x}, u_{i}^{y}]^{\top}$ is the velocity of the fluid; $p_i$ is the pressure;  $\nu_i$ is the viscosity parameters; $\mathbb{I}$ is the identity matrix; $\boldsymbol{n}_i$ is unit outer normals to the subdomains $\Omega_i$ at the interface $\Gamma$, respectively; 
$\boldsymbol{\tau}$ is any unit tangential vector to the interface $\Gamma$; $\kappa>0$ is the traction coefficient.
To generate the training and testing datasets to illustrate the effectiveness of the proposed PISAL,  the following setups are given:
$\Omega_1=[0,1]\times[0,1]$, $\Omega_2=[0,1]\times[-1,0]$, 
$\Gamma=y=0$, $\nu_1=0.1$, $\nu_2=0.2$, $\kappa=0.1$, and $\sigma=10$.
The coupled Stokes-Stokes system has the following exact solutions
\begin{equation}
\begin{aligned}
& u_1^x=-x^2(x-1)^2(y-1), \\
& u_1^y=x y\left(6 x+y-3 x y+2 x^2 y-4 x^2-2\right), \\
& p_1=\cos (\pi x) \sin (\pi y), \\
& u_2^x=-\frac{\nu_1}{\nu_2} x^2(x-1)^2 y+\left(\frac{\nu_1}{\kappa}+1\right) x^2(x-1)^2, \\
& u_2^y=\hspace{-1mm}(2 x-1) \frac{\nu_1}{\nu_2} x(x-1) y^2\hspace{-0.5mm}-\hspace{-0.5mm}(2 x-1)\left(2 \frac{\nu_1}{\kappa}+2\right) x(x-1) y, \\
& p_2=\cos (\pi x) \sin (\pi y) .
\end{aligned}
\end{equation} 
The force terms $\boldsymbol{g}_i $ and the boundary conditions are then calculated by inserting the exact solutions into \eqref{eq:23112201}.
Considering the practical situations, the exact interface $\Gamma$ and the viscosity parameters $\nu_i$ are assumed to be unknown.

In this case, 
 both $Net_1$ and $Net_2$ consist of two hidden layers with 90 neurons individually; $Net_I$ consists of one hidden layer with 100 neurons.
A total number of training data $N_u=1000$ and collocation points $N_f=1000$ are randomly sampled from the whole domain 
$\Omega$.
The magnitude of predictions $\boldsymbol{\hat u}_i(x,y)$ and $\hat p_i(x, y)$ are shown in Fig. \ref{fig:NS}(a) and Fig. \ref{fig:NS}(b), respectively; where $\boldsymbol{\hat u}_i(x,y)=\sqrt {\hat {{u}_i^{x}}^{2}+\hat {{u}_i^{y}}^{2}}$.
\begin{figure}[h]
\begin{center}
\includegraphics[width=7cm]{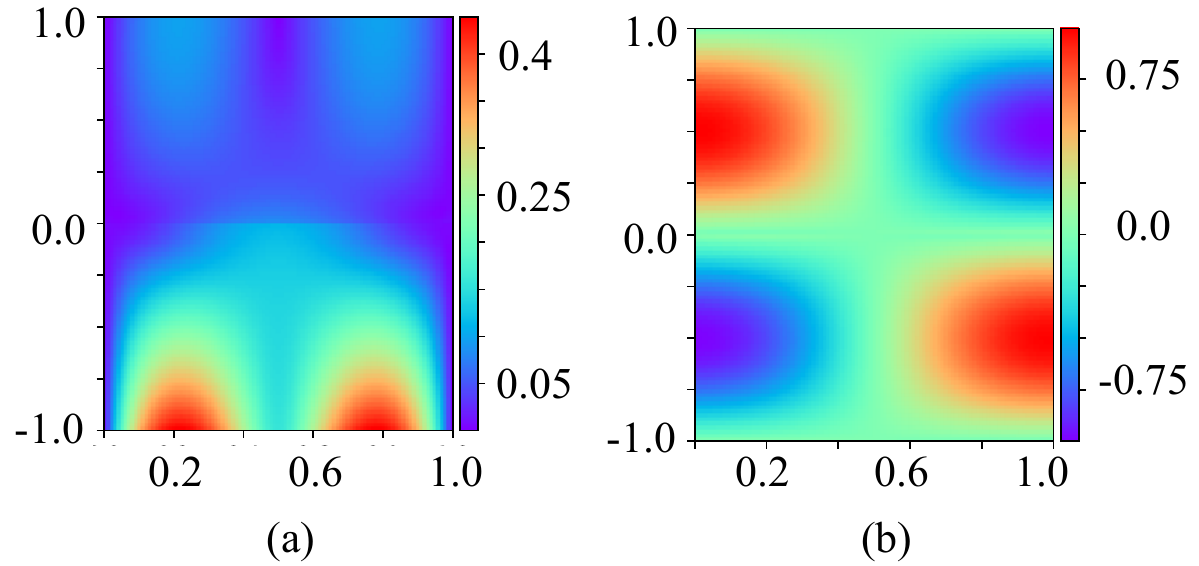}
\caption{Predictions for Navier-Stokes. (a) Velocity. (b) Pressure.}
\label{fig:NS}
\end{center}
\end{figure}
Moreover, the friction-type interface $\Gamma$ can be accurately identified. 
Similar to Section \ref{section:Two-Phase Stefan Problem}, PINNs and indicator-PINNs
are conducted as benchmarks to evaluate  the performance of the proposed PISAL. The best values are bold in Table \ref{tb:NS}, which indicate our proposed PISAL outperforms the benchmarks. 
\begin{table*}[h]
\begin{center}
\caption{Comparisons for mixed Navier-Stokes}
\label{tb:NS}
\setlength{\tabcolsep}{2mm}{
\begin{tabular}{cccccc|ccc}
\hline
\tabincell{c}{Method}&
$\Omega_i$
&
\tabincell{c}{Criteria} &
\tabincell{c}{$u_i^x$} &
\tabincell{c}{$u_i^y$} &
\tabincell{c}{$p$} &
\tabincell{c}{$\nu_i$}
&
\tabincell{c}{PE(\%)}
\\\hline
\multirow{4}{*}{\tabincell{c}{PISAL}}&\multirow{2}{*}{\tabincell{c}{$\Omega_1$}}&{\tabincell{c}{${\rm RMSE}$}}&   \textbf{2.343{\it e}-03} & \textbf{9.116{\it e}-04}& \textbf{7.387{\it e}-03}&\multirow{2}{*}{\textbf{0.0989}} & \multirow{2}{*}{\textbf{0.112}}\\
& &{\tabincell{c}{CC}} & \textbf{9.999{\it e}-01} & \textbf{9.999{\it e}-01}& \textbf{9.999{\it e}-01}& & \\  \cmidrule(r){2-8} 
&\multirow{2}{*}{\tabincell{c}{$\Omega_2$}}&{\tabincell{c}{${\rm RMSE}$}}&   \textbf{2.218{\it e}-03} & \textbf{7.111{\it e}-04}& \textbf{7.311{\it e}-04}&\multirow{2}{*}{\textbf{0.1982}}&\multirow{2}{*}{\textbf{0.179}}  \\
& &{\tabincell{c}{CC}} & \textbf{9.999{\it e}-01} & \textbf{9.999{\it e}-01}& \textbf{9.999{\it e}-01}& &\\\hline
\multirow{4}{*}{\tabincell{c}{PINNs}}&\multirow{2}{*}{\tabincell{c}{$\Omega_1$}}&{\tabincell{c}{${\rm RMSE}$}}&   5.041{\it e}-03 & 4.631{\it e}-03& 2.742{\it e}-02&\multirow{2}{*}{0.0946} & \multirow{2}{*}{10.543}\\
& &{\tabincell{c}{CC}} & 9.736{\it e}-01 & 9.967{\it e}-01& 9.985{\it e}-01& & \\  \cmidrule(r){2-8} 
&\multirow{2}{*}{\tabincell{c}{$\Omega_2$}}&{\tabincell{c}{${\rm RMSE}$}}&   5.501{\it e}-03 & 2.364{\it e}-03& 1.256{\it e}-02&\multirow{2}{*}{0.0946}&\multirow{2}{*}{0.543}  \\
& &{\tabincell{c}{CC}} & 9.944{\it e}-01 & 9.999{\it e}-01& 9.996{\it e}-01& &\\\hline
\multirow{4}{*}{\tabincell{c}{indicator-\\PINNs}}&\multirow{2}{*}{\tabincell{c}{$\Omega_1$}}&{\tabincell{c}{${\rm RMSE}$}}&   2.685{\it e}-03 &2.463{\it e}-03& 3.827{\it e}-03&\multirow{2}{*}{0.1111} & \multirow{2}{*}{1.111}\\
& &{\tabincell{c}{CC}} & 9.988{\it e}-01 & 9.999{\it e}-01& 9.997{\it e}-01& & \\  \cmidrule(r){2-8} 
&\multirow{2}{*}{\tabincell{c}{$\Omega_2$}}&{\tabincell{c}{${\rm RMSE}$}}&   2.393{\it e}-03 &1.602{\it e}-03& 1.119{\it e}-03&\multirow{2}{*}{0.1936} & \multirow{2}{*}{0.314}\\
& &{\tabincell{c}{CC}} & 9.988{\it e}-01 & 9.999{\it e}-01& 9.997{\it e}-01& &\\\hline
\end{tabular}}
\end{center}
\end{table*}
\vspace{-2mm}
\subsection{Discussion}\label{Discussion}
\subsubsection{Approximation ability}
The  approximation ability of the proposed PISAL has been theoretically proved in Section \ref{Subsection II-D}.
Take the Stefan problem to illustrate the practical applications. In this case, the value of ${{\rm {MSE}}_M}$ is (1.634{\it e}-06+1.047{\it e}-06+1.439{\it e}-08)+(9.398{\it e}-07+1.130{\it e}-06+6.084{\it e}-08)  when meeting the stop criterion in Algorithm \ref{alg:algorithm-label}.
It can be seen that there exists a set of $\left\{{\boldsymbol{\Theta}}_{U_1}^{(j)},{\boldsymbol{\Theta}}_{U_2}^{(j)},{\boldsymbol{\Theta}}_{U_I}^{(j)}\right\}$ to realize the loss $L_{M} \left({\boldsymbol{\Theta}}_{U_1}^{(j)},{\boldsymbol{\Theta}}_{U_2}^{(j)},{\boldsymbol{\Theta}}_{U_I}^{(j)}\right)\rightarrow 0\ (j\rightarrow\infty)$.  
The temperature distributions satisfying \eqref{eq:231120-07} and \eqref{eq:231120-08}  can be obtained when the proposed PISAL is well-trained achieving $\left\{\hat{{\boldsymbol{\Theta}}}_{U_1}^{(j)},{\hat{\boldsymbol{\Theta}}}_{U_2}^{(j)},{\hat{\boldsymbol{\Theta}}}_{U_I}^{(j)}\right\}$ by using the SAL strategy. 
Furthermore, $\Vert\hat{u}_i\left(\boldsymbol{x}, t, \hat{\boldsymbol{\Theta}}_{U_i}^{(j)}; k_i\right)-u_i\Vert<\varepsilon$ for all $\varepsilon>0$ can be obtained according to the property of $u_i$ in Theorem \ref{thm1} and Remark \ref{231217}.
That is,  the thermal diffusivity parameters $k_i$ and the time-varying interface $s(t)$ can be synchronously obtained.
\subsubsection{Complexity analysis}
Complexity analysis is a common and comparable performance criterion used to express
the computational efficiency of a model\cite{geng2021novel, 9701488}.
To further evaluate the efficiency performance, PINNs and indicator-PINNs are conducted as benchmarks.
We use the big $O$ notation method to express the computational complexity.
The complexity analyses of the proposed PISAL and 
benchmarks are shown in Table \ref{tb:complexity}, which indicate the comparison methods are in similar scales.
\begin{table}[h]
\begin{center}\vspace{-1mm}
\caption{Comparisons for complexity analysis}
\vspace{-2mm}
\label{tb:complexity}
\setlength{\tabcolsep}{0.5mm}{
\begin{tabular}{ccccccc}
\hline
\tabincell{c}{Method}&
\tabincell{c}{$Net_1$} &
\tabincell{c}{$Net_2$} &
\tabincell{c}{$Net_I$} &
\tabincell{c}{Computational complexity}
\\\hline
{\tabincell{c}{PISAL}}&$O(l_1  d_1)$&  $O(l_2  d_2)$& $O(l_3  d_3)$&$O(l_1  d_1+l_2  d_2+l_3  d_3)$ \\\hline
PINNs&$O(l_1  d_1)$&--&--&$O(l_1  d_1)$
\\\hline
{\tabincell{c}{indicator-\\PINNs}}&$O(l_1  d_1)$& --& $O(l_3  d_3)$&$O(l_1  d_1+l_3  d_3)$\\\hline
\end{tabular}}
\end{center}
\end{table}
Here $l_1$, $l_2$, and $l_3$ represent the number of $Net_1$ layers, $Net_2$ layers, and $Net_I$ layers, respectively.
 $d_1$, $d_2$, and $d_3$ represent the number of $Net_1$ hidden vector dimension, $Net_2$ hidden vector dimension, and $Net_I$  hidden vector dimension, respectively.
Note that  PINNs and its variants solve PDEs in small data regimes.
Thus, the performance with respect to model accuracy is our investigative focus when the difference with respect to model complexity is indistinguishable and ample computing power is available.
Note that predictions can be obtained by using the well-trained PISAL on our computer in less than 1 second. 
\subsubsection{Impact of Neural Network  Size and Training Data Number on Performance}\label{size and number}
The performance of our proposed PISAL is further scrutinized by conducting tests with respect to the different network sizes and the total number of training data.
Taking the mixed Navier-Stokes problem as an example, the results are summarized in Tables  \ref{tb:networkSize} and
 \ref{tb:numberData}.
\begin{table*}[htbp]
 \begin{minipage}{\textwidth}
\begin{center}
\caption{PE in identified parameters $\hat \nu_1$ and $\hat\nu_2$ for different numbers of hidden layers and neurons per layer. Here, the training data and collocation points are considered to be  fixed to $N_u = 1000$ and $N_f=2000$.}
\vspace{2mm}
\label{tb:networkSize}
\setlength{\tabcolsep}{1.5mm}{
\begin{tabular}{c|ccccccccccccccc}
\hline
&\multicolumn{4}{c}{$\nu_1 (\%)$}&\multicolumn{4}{c}{$\nu_2 (\%)$}
\\ \cmidrule(r){1-5} \cmidrule(r){6-9}
\diagbox{Layers}{Neurons}&
\tabincell{c}{70} &
\tabincell{c}{90} &
\tabincell{c}{105} &
\tabincell{c}{120}&
\tabincell{c}{70} &
\tabincell{c}{90} &
\tabincell{c}{105} &
\tabincell{c}{120}
\\ \cmidrule(r){1-5} \cmidrule(r){6-9}
{1}&  0.4324 & 0.3249 & 0.0187 &  0.1127 &0.3912 & 0.0381&  0.1553 &0.1081 &    \\ \cmidrule(r){1-5} \cmidrule(r){6-9} 
2&   0.1853 &  0.1790  &0.0371  &0.0899&0.2671 &0.112 &0.1184 &0.0256 \\ \cmidrule(r){1-5} \cmidrule(r){6-9} 
3& 0.0717 &  0.0999 & 0.1403&0.0570 &0.3380&0.2554  &0.0071 & 0.0289 \\ \cmidrule(r){1-5} \cmidrule(r){6-9} 
4& 0.1290 & 0.1276 &0.1156&0.0363&0.0742&0.0656&0.0027&0.1059     \\\hline 
\end{tabular}}
\end{center}
 \end{minipage}
\end{table*}
\begin{table*}[htbp]
 \begin{minipage}{\textwidth}
\begin{center}
\caption{PE in identified parameters $\hat\nu_1$ and $\hat\nu_2$ for different numbers of training data $N_u$ and $N_f$. Here, the neural network size is kept fixed to two layers and 90 neurons per layer.}
\vspace{2mm}
\label{tb:numberData}
\setlength{\tabcolsep}{1.8mm}{
\begin{tabular}{c|ccccccccccccccc}
\hline&\multicolumn{4}{c}{$\nu_1 (\%)$}&\multicolumn{4}{c}{$\nu_2 (\%)$}\\ \cmidrule(r){1-5} \cmidrule(r){6-9} 
\diagbox{$N_u$}{$N_f$}&500&1000&1500&2000&500&1000&1500&2000
\\ \cmidrule(r){1-5} \cmidrule(r){6-9} 
{500}&   0.3337 & 0.1786  &0.2004 & 0.2412&  0.5279 & 0.5913  &0.6419 &0.5650  \\ \cmidrule(r){1-5} \cmidrule(r){6-9} 
1000& 0.3256 & 0.2655 &  0.1012&0.1532&0.8829&0.4312&0.5834&0.4364  \\ \cmidrule(r){1-5} \cmidrule(r){6-9}  
1500&   0.5017 & 0.2137&  0.1462&0.2311&0.3421&0.2974 &0.3559&0.3071 \\ \cmidrule(r){1-5} \cmidrule(r){6-9}  
2000&  0.3706 & 0.1057 & 0.2252 &0.2830&0.6146&0.5401&0.7854&0.8329  \\\hline
\end{tabular}}
\end{center}
 \end{minipage}
\end{table*}The performance of the proposed PISAL with respect to different network sizes is shown in Table \ref{tb:networkSize}.
Here, the training data and collocation points are fixed to $N_u = 1000$ and $N_f = 2000$, respectively.
Note that the widening and deepening of the networks
result in identification performance improvement.
However,
 the trend is not strictly monotonous, which could be attributed to numerous reasons, such as the equation itself as well as the particular neural network setups.
 The performance of the proposed PISAL with respect to different numbers of training data $N_u$ and collocation points $N_f$ is shown in Table \ref{tb:numberData}. Here, the network size are fixed to two layers with 90 neurons per layer. 
There are some variabilities and non-monotonic trends, which could be attributed to the neural network nonconvexity or some factors pertaining to the network structure or equation itself.
\section{Conclusion}\label{Conclusion}
This article advocated for applying PISAL to address the intricacies of industrial systems modeling, including unknown PDE parameters due to a lack of physical attributions and an unavailable  time-varying interface due to operating in heterogeneous media. 
First,
$Net_1$, $Net_2$, and  $Net_I$ is constructed for approximating the soutions satisfying PDEs and the time-varying interface.
$Net_1$ and $Net_2$ are utilized to synchronously learn each solution satisfying PDEs with diverse parameters, while 
$Net_I$ is employed to adaptively learn the unavaliable time-varying interface.
Then, 
a criterion combined with $Net_I$  is proposed, which is used to adaptively distinguish the attributions of measurements and colocation points.
Furthermore, $Net_I$ is integrated into a data-physics-hybrid loss function.
Accordingly, a SAL strategy is proposed for decomposing and optimizing  each subdomain.
Besides, the approximation capability of the proposed PISAL is theoretically proved.
Finally, two illustrative examples are conducted to validate the effectiveness and feasibility of our proposed method. Meanwhile, some comparisons with relevant state-of-the-art methods and discussions are given.
We believe that the proposed PISAL can be expected to benefit practitioners, such as nondestructive testing, soft sensing, and model predictive control.
Note that the networks of PISAL are not limited to  FCNNs, convolutional neural networks, recurrent neural networks, etc., which are determined by the practical situations.

In the future, we will continue to use PISAL for more complex practical situations, such as unknown-driven sources.
Meanwhile, we intend to conduct the proposed PISAL on an experimental platform to further verify the feasibility in practical scenarios.
Some variabilities and non-monotonic trends, shown in Tables \ref{tb:networkSize} and \ref{tb:numberData},  are also worth further investigation, as we discussed in Section \ref{size and number}.
\bibliographystyle{IEEEtran}
\bibliography{myref}

\begin{thebibliography}{10}
\providecommand{\url}[1]{#1}
\csname url@samestyle\endcsname
\providecommand{\newblock}{\relax}
\providecommand{\bibinfo}[2]{#2}
\providecommand{\BIBentrySTDinterwordspacing}{\spaceskip=0pt\relax}
\providecommand{\BIBentryALTinterwordstretchfactor}{4}
\providecommand{\BIBentryALTinterwordspacing}{\spaceskip=\fontdimen2\font plus
\BIBentryALTinterwordstretchfactor\fontdimen3\font minus
  \fontdimen4\font\relax}
\providecommand{\BIBforeignlanguage}[2]{{%
\expandafter\ifx\csname l@#1\endcsname\relax
\typeout{** WARNING: IEEEtran.bst: No hyphenation pattern has been}%
\typeout{** loaded for the language `#1'. Using the pattern for}%
\typeout{** the default language instead.}%
\else
\language=\csname l@#1\endcsname
\fi
#2}}
\providecommand{\BIBdecl}{\relax}
\BIBdecl

\bibitem{10304675}
N.~Ji and J.~Liu, ``Quantized time-varying fault-tolerant control for a
  three-dimensional euler–bernoulli beam with unknown control directions,''
  \emph{IEEE Transactions on Automation Science and Engineering}, pp. 1--12,
  2023.

\bibitem{9516587}
S.~Zhang, Y.~Wu, X.~He, and Z.~Liu, ``Cooperative fault-tolerant control for a
  mobile dual flexible manipulator with output constraints,'' \emph{IEEE
  Transactions on Automation Science and Engineering}, vol.~19, no.~4, pp.
  2689--2698, 2022.

\bibitem{yu2022hybrid}
H.~Yu, J.~Gong, G.~Wang, and X.~Chen, ``A hybrid model for billet tapping
  temperature prediction and optimization in reheating furnace,'' \emph{IEEE
  Transactions on Industrial Informatics}, 2022.

\bibitem{feng2023computation}
Y.~Feng, Y.~Wang, Y.~Mo, Y.~Jiang, Z.~Liu, W.~He, and H.-X. Li,
  ``Computation-efficient fault detection framework for partially known
  nonlinear distributed parameter systems,'' \emph{IEEE Transactions on Neural
  Networks and Learning Systems}, 2023.

\bibitem{ji2022vibration}
N.~Ji and J.~Liu, ``Vibration control for a three-dimensional variable length
  flexible string with time-varying actuator faults and unknown control
  directions,'' \emph{IEEE Transactions on Automation Science and Engineering},
  2022.

\bibitem{huang2023physical}
K.~Huang, S.~Tao, D.~Wu, C.~Yang, and W.~Gui, ``Physical informed sparse
  learning for robust modeling of distributed parameter system and its
  industrial applications,'' \emph{IEEE Transactions on Automation Science and
  Engineering}, 2023.

\bibitem{5345678}
W.~Huang and Z.~Kong, ``Process capability sensitivity analysis for design
  evaluation of multistage assembly processes,'' \emph{IEEE Transactions on
  Automation Science and Engineering}, vol.~7, no.~4, pp. 736--745, 2010.

\bibitem{9852677}
B.~Petrus, Z.~Chen, H.~El-Kebir, J.~Bentsman, and B.~G. Thomas, ``Solid
  boundary output feedback control of the stefan problem: The enthalpy
  approach,'' \emph{IEEE Transactions on Automatic Control}, vol.~68, no.~6,
  pp. 3485--3500, 2023.

\bibitem{10234210}
Z.~Li and D.~Shen, ``Filter-free parameter estimation method for
  continuous-time systems,'' \emph{IEEE Transactions on Automation Science and
  Engineering}, pp. 1--16, 2023.

\bibitem{9462124}
L.~Qiu and H.~Ren, ``Rsegnet: A joint learning framework for deformable
  registration and segmentation,'' \emph{IEEE Transactions on Automation
  Science and Engineering}, vol.~19, no.~3, pp. 2499--2513, 2022.

\bibitem{9695519}
J.~Xie and B.~Yao, ``Physics-constrained deep learning for robust inverse ecg
  modeling,'' \emph{IEEE Transactions on Automation Science and Engineering},
  vol.~20, no.~1, pp. 151--166, 2023.

\bibitem{wang2019spatiotemporal}
D.~Wang, K.~Liu, and X.~Zhang, ``Spatiotemporal thermal field modeling using
  partial differential equations with time-varying parameters,'' \emph{IEEE
  Transactions on Automation Science and Engineering}, vol.~17, no.~2, pp.
  646--657, 2019.

\bibitem{4049770}
Y.-M.~H. Ng, M.~Yu, Y.~Huang, and R.~Du, ``Diagnosis of sheet metal stamping
  processes based on 3-d thermal energy distribution,'' \emph{IEEE Transactions
  on Automation Science and Engineering}, vol.~4, no.~1, pp. 22--30, 2007.

\bibitem{9423982}
H.~Yu and Z.~Hua, ``Dynamic sampling policy for in situ and online measurements
  data fusion in a policy network,'' \emph{IEEE Transactions on Automation
  Science and Engineering}, vol.~19, no.~3, pp. 2016--2029, 2022.

\bibitem{sun2021survey}
Q.~Sun and Z.~Ge, ``A survey on deep learning for data-driven soft sensors,''
  \emph{IEEE Transactions on Industrial Informatics}, vol.~17, no.~9, pp.
  5853--5866, 2021.

\bibitem{10268910}
N.~Yassaie, S.~Gargoum, and A.~W. Al-Dabbagh, ``Data-driven fault
  classification in large-scale industrial processes using reduced number of
  process variables,'' \emph{IEEE Transactions on Automation Science and
  Engineering}, pp. 1--0, 2023.

\bibitem{sirignano2018dgm}
J.~Sirignano and K.~Spiliopoulos, ``Dgm: A deep learning algorithm for solving
  partial differential equations,'' \emph{Journal of computational physics},
  vol. 375, pp. 1339--1364, 2018.

\bibitem{lu2021deepxde}
L.~Lu, X.~Meng, Z.~Mao, and G.~E. Karniadakis, ``Deepxde: A deep learning
  library for solving differential equations,'' \emph{SIAM review}, vol.~63,
  no.~1, pp. 208--228, 2021.

\bibitem{yao2021figan}
Z.~Yao and C.~Zhao, ``Figan: A missing industrial data imputation method
  customized for soft sensor application,'' \emph{IEEE Transactions on
  Automation Science and Engineering}, vol.~19, no.~4, pp. 3712--3722, 2021.

\bibitem{lui2022supervised}
C.~F. Lui, Y.~Liu, and M.~Xie, ``A supervised bidirectional long short-term
  memory network for data-driven dynamic soft sensor modeling,'' \emph{IEEE
  Transactions on Instrumentation and Measurement}, vol.~71, pp. 1--13, 2022.

\bibitem{10061539}
Z.~Hu, C.~Wang, J.~Wu, and D.~Du, ``Gaussian process latent variable
  model-based multi-output modeling of incomplete data,'' \emph{IEEE
  Transactions on Automation Science and Engineering}, pp. 1--11, 2023.

\bibitem{raissi2019physics}
M.~Raissi, P.~Perdikaris, and G.~E. Karniadakis, ``Physics-informed neural
  networks: A deep learning framework for solving forward and inverse problems
  involving nonlinear partial differential equations,'' \emph{Journal of
  Computational physics}, vol. 378, pp. 686--707, 2019.

\bibitem{9931982}
J.~Zhao, W.~Li, X.~Yuan, X.~Yin, X.~Li, Q.~Chen, and J.~Ding, ``An end-to-end
  physics-informed neural network for defect identification and 3-d
  reconstruction using rotating alternating current field measurement,''
  \emph{IEEE Transactions on Industrial Informatics}, vol.~19, no.~7, pp.
  8340--8350, 2023.

\bibitem{huang2022physics}
G.~Huang, Z.~Zhou, F.~Wu, and W.~Hua, ``Physics-informed time-aware neural
  networks for industrial nonintrusive load monitoring,'' \emph{IEEE
  Transactions on Industrial Informatics}, 2022.

\bibitem{jagtap2022physics}
A.~D. Jagtap, Z.~Mao, N.~Adams, and G.~E. Karniadakis, ``Physics-informed
  neural networks for inverse problems in supersonic flows,'' \emph{Journal of
  Computational Physics}, vol. 466, p. 111402, 2022.

\bibitem{zhang2022multi}
B.~Zhang, G.~Wu, Y.~Gu, X.~Wang, and F.~Wang, ``Multi-domain physics-informed
  neural network for solving forward and inverse problems of steady-state heat
  conduction in multilayer media,'' \emph{Physics of Fluids}, vol.~34, no.~11,
  2022.

\bibitem{alhubail2022extended}
A.~Alhubail, X.~He, M.~AlSinan, H.~Kwak, and H.~Hoteit, ``Extended
  physics-informed neural networks for solving fluid flow problems in highly
  heterogeneous media,'' in \emph{International Petroleum Technology
  Conference}.\hskip 1em plus 0.5em minus 0.4em\relax IPTC, 2022, p.
  D031S073R001.

\bibitem{aliakbari2023ensemble}
M.~Aliakbari, M.~Soltany~Sadrabadi, P.~Vadasz, and A.~Arzani, ``Ensemble
  physics informed neural networks: A framework to improve inverse transport
  modeling in heterogeneous domains,'' \emph{Physics of Fluids}, vol.~35,
  no.~5, 2023.

\bibitem{9632211}
V.~Chernyshov, A.~Kaz’min, and P.~Fedyunin, ``Testing electrophysical
  parameters of multilayer dielectric and magnetodielectric coatings by the
  method of surface electromagnetic waves,'' in \emph{2021 3rd International
  Conference on Control Systems, Mathematical Modeling, Automation and Energy
  Efficiency (SUMMA)}, 2021, pp. 372--377.

\bibitem{cai2021physics}
S.~Cai, Z.~Wang, S.~Wang, P.~Perdikaris, and G.~E. Karniadakis,
  ``Physics-informed neural networks for heat transfer problems,''
  \emph{Journal of Heat Transfer}, vol. 143, no.~6, p. 060801, 2021.

\bibitem{li2021schwarz}
W.~Li and Y.~Xu, ``Schwarz domain decomposition methods for the fluid-fluid
  system with friction-type interface conditions,'' \emph{Applied Numerical
  Mathematics}, vol. 166, pp. 114--126, 2021.

\bibitem{wang2023coupled}
A.~Wang, P.~Qin, and X.-M. Sun, ``Coupled physics-informed neural networks for
  inferring solutions of partial differential equations with unknown source
  terms,'' \emph{arXiv preprint arXiv:2301.08618}, 2023.

\bibitem{mishra2020estimates}
S.~Mishra and R.~Molinaro, ``Estimates on the generalization error of physics
  informed neural networks (pinns) for approximating a class of inverse
  problems for pdes,'' \emph{arXiv preprint arXiv:2007.01138}, 2020.

\bibitem{baydin2018automatic}
A.~G. Baydin, B.~A. Pearlmutter, A.~A. Radul, and J.~M. Siskind, ``Automatic
  differentiation in machine learning: a survey,'' \emph{Journal of Marchine
  Learning Research}, vol.~18, pp. 1--43, 2018.

\bibitem{gomez2019review}
H.~Gomez, M.~Bures, and A.~Moure, ``A review on computational modelling of
  phase-transition problems,'' \emph{Philosophical Transactions of the Royal
  Society A}, vol. 377, no. 2143, p. 20180203, 2019.

\bibitem{kim2021stefan}
I.~C. Kim and Y.-H. Kim, ``The stefan problem and free targets of optimal
  brownian martingale transport,'' \emph{arXiv preprint arXiv:2110.03831},
  2021.

\bibitem{nandi2022second}
S.~Nandi and Y.~Sanyasiraju, ``A second order accurate fixed-grid method for
  multi-dimensional stefan problem with moving phase change materials,''
  \emph{Applied Mathematics and Computation}, vol. 416, p. 126719, 2022.

\bibitem{wang2021deep}
S.~Wang and P.~Perdikaris, ``Deep learning of free boundary and stefan
  problems,'' \emph{Journal of Computational Physics}, vol. 428, p. 109914,
  2021.

\bibitem{geng2021novel}
Z.~Geng, Z.~Chen, Q.~Meng, and Y.~Han, ``Novel transformer based on gated
  convolutional neural network for dynamic soft sensor modeling of industrial
  processes,'' \emph{IEEE Transactions on Industrial Informatics}, vol.~18,
  no.~3, pp. 1521--1529, 2021.

\bibitem{9701488}
J.~Zhang, D.~Zhou, M.~Chen, and X.~Hong, ``Continual learning for multimode
  dynamic process monitoring with applications to an ultra–supercritical
  thermal power plant,'' \emph{IEEE Transactions on Automation Science and
  Engineering}, vol.~20, no.~1, pp. 137--150, 2023.

\end{thebibliography}

\end{document}